\documentclass[11pt]{article}

\usepackage[a4paper,margin=1in]{geometry}
\usepackage[utf8]{inputenc}
\usepackage[T1]{fontenc}
\usepackage{lmodern}
\usepackage{microtype}
\usepackage{amsmath,amssymb,amsfonts,amsthm,mathtools}
\usepackage{mathrsfs}
\IfFileExists{bbm.sty}{\usepackage{bbm}}{}
\usepackage{enumitem}
\usepackage{xcolor}
\usepackage{graphicx}
\usepackage{booktabs,tabularx,array}
\usepackage{hyperref}
\usepackage[nameinlink,capitalise]{cleveref}

\hypersetup{
  colorlinks=true,
  linkcolor=blue!60!black,
  citecolor=blue!60!black,
  urlcolor=blue!60!black
}

\newcommand{\DOI}[1]{\href{https://doi.org/#1}{\nolinkurl{doi:#1}}}

\newtheorem{theorem}{Theorem}[section]
\newtheorem{proposition}[theorem]{Proposition}

\newtheorem{corollary}[theorem]{Corollary}
\newtheorem{assumption}[theorem]{Assumption}
\newtheorem{definition}[theorem]{Definition}
\newtheorem{remark}[theorem]{Remark}

\newcommand{\R}{\mathbb{R}}

\newcommand{\E}{\mathbb{E}}

\newcommand{\dd}{\,\mathrm{d}}

\newcommand{\Law}{\operatorname{Law}}
\IfFileExists{bbm.sty}{\newcommand{\one}{\mathbbm{1}}}{\newcommand{\one}{\mathbf{1}}}

\newcommand{\includefigureorplaceholder}[2][]{%
  \IfFileExists{#2}{\includegraphics[#1]{#2}}{%
    \fbox{\parbox[c][0.22\textheight][c]{0.82\linewidth}{%
      \centering
      Figure file \texttt{\detokenize{#2}} was not supplied with the source.\\
      Add it beside this \texttt{.tex} file to restore the original figure.}}}}

\title{An Adjoint-Sensitivity Framework for Lost-in-the-Middle Phenomena in Causal Residual Transformers}

\makeatletter

\def\author#1#2{\expandafter\gdef\csname transformer@author#1\endcsname{#2}\transformer@setauthors}
\def\address#1#2{\expandafter\gdef\csname transformer@address#1\endcsname{#2}\transformer@setauthors}
\def\email#1#2{\expandafter\gdef\csname transformer@email#1\endcsname{#2}\transformer@setauthors}
\newcommand{\transformer@authorentry}[1]{%
  \csname transformer@author#1\endcsname\thanks{%
  \csname transformer@address#1\endcsname. Email: \href{mailto:\csname transformer@email#1\endcsname}{\texttt{\csname transformer@email#1\endcsname}}}}
\newcommand{\transformer@setauthors}{%
  \gdef\@author{%
    \@ifundefined{transformer@author1}{}{\transformer@authorentry{1}}%
    \@ifundefined{transformer@author2}{}{\and\transformer@authorentry{2}}%
  }%
}
\makeatother
\author1{Cheng Huan}
\address1{Department of Statistics and Data Science, The Chinese University of Hong Kong}
\email1{chenghuan@cuhk.edu.hk}
\author2{Hongwei Yuan}
\address2{Department of Mathematics, University of Macau}
\email2{hwyuan@um.edu.mo}
\date{}

\begin{document}
\maketitle

\begin{abstract}
We develop an adjoint-sensitivity framework for positional influence in causal residual Transformers and separate unconditional analytic results from conditional boundary-shape conclusions. The principal unconditional theorem is the residual-to-depth-flow estimate for layer controls converging in $L^1$, complemented by a finite-token-to-Volterra attention estimate that explicitly controls the first cells near the causal endpoint. We define a normalized adjoint-energy influence density and derive its exact evolution along full-batch gradient flow. The adjoint admits an exact generator-term decomposition into residual transmission, nonlocal Volterra, and local channels, including all covariance cross terms. Causal masking can amplify early-position sensitivity and residual identity paths can transmit a right-localized terminal bias, but neither mechanism alone forces a U-shaped profile. We therefore state boundary advantages under independently checkable energy, correlation, and local-channel bounds; these conditions are sufficient rather than necessary. Finite-token influence balancing, positional reweighting, and task-aligned observability are presented as diagnostics or regularizers with explicit differentiation requirements, computational costs, and limitations. Controlled simulations illustrate that each intervention controls its designated surrogate, while observability balance or outer-loop reweighting need not monotonically reduce the influence-based Lost-in-the-Middle diagnostic.
\end{abstract}

\noindent\textbf{Keywords:} positional influence density; Lost-in-the-Middle; primacy and recency; position bias; deterministic influence control; deterministic gradient flow; continuous-depth causal residual Transformers; causal masking; Volterra adjoint; observability regularization.\par
\smallskip
\noindent\textbf{MSC 2020:} 68T07; 49K15; 65L20.

\section{Introduction}
Long-context language models can accept large prompts without using every position equally. Empirical work documents substantial position sensitivity in retrieval and question answering, often with stronger performance near the beginning or end of a context than in its middle \cite{hsieh2024found,liu2024lost}. We use \emph{Lost-in-the-Middle} for this double boundary advantage and ask a narrower and purely theoretical question:
\begin{quote}
Can primacy, recency, and lost-in-the-middle be derived as consequences of deterministic continuous-time training dynamics and continuous-depth causal residual dynamics, and can one write principled influence-control laws that remove them?
\end{quote}

Our answer is affirmative in a precise but conditional sense.  We do not claim that the model reproduces every feature of a production Transformer or that one mechanism explains every long-context failure.  Rather, the paper first establishes the analytic framework in which positional sensitivity can be compared: the residual recursion converges to the continuous-depth flow when the interpolated layer controls converge in $L^1$, as quantified by Theorem \ref{thm:resnet-ode}, finite-token causal attention converges to the Volterra operator by Theorem \ref{thm:attention-continuum-error}, and the backward equation \eqref{eq:adjoint} defines the adjoint-energy influence measure in Definition \ref{def:influence}.  Proposition \ref{prop:influence-differentiability} then justifies differentiating its density along the gradient flow \eqref{eq:gd-flow}, leading to the exact evolution identity \eqref{eq:gd-mdot}.

Within this framework, causal masking and residual propagation provide two distinct \emph{available} channels.  The Volterra adjoint formula \eqref{eq:volterra-adjoint} transports covectors from later outputs to earlier input positions; Proposition \ref{prop:primacy-noncancellation} yields a quantitative cone-channel primacy effect under joint cross-data and cross-depth coercivity, together with a middle-channel upper bound that rules out the comparison-of-lower-bounds fallacy.  The residual Duhamel identity \eqref{eq:continuous-adjoint-duhamel}, together with the localized stability estimate \eqref{eq:residual-stability-bound}, shows that a right-localized terminal adjoint can persist to depth zero, but does not assert that the residual path creates such localization by itself.  Proposition \ref{prop:regional-recency-persistence} gives a dimensionally consistent squared-energy condition for this persistence.  Finally, the exact regional decomposition \eqref{eq:exact-regional-influence-decomposition} retains all channel cross-covariances, and Theorem \ref{thm:channel-energy-double-deficit} gives explicit sufficient conditions under which the middle average is smaller than both boundary averages, equivalently when the index \eqref{eq:gd-lim-index} is positive.

The unconditional results concern approximation, adjoint propagation, and influence-measure dynamics, whereas primacy, recency, and a U-shaped lost-in-the-middle profile are conditional mechanism conclusions.  This distinction is essential: causal geometry enlarges the future cone of early positions and residual dynamics preserve terminal sensitivity, but neither feature alone forces a U-shaped influence profile for every parameter state, readout, and data law.

\subsection{Relation to prior work and novelty}

\paragraph{Gradient attribution and influence in sequence models.}
Input gradients and integrated gradients quantify prediction sensitivity to input features \cite{sundararajan2017ig}; attention rollout and attention flow instead propagate attention-based relevance across layers \cite{abnar2020flow}. Our density is a normalized squared input-adjoint energy averaged over the data law. It is not an attention-weight explanation and does not satisfy the path axioms of integrated gradients. Its role is to provide a position-indexed sensitivity measure compatible with a forward--backward depth system and with exact regional energy identities.

\paragraph{Adjoint sensitivity and continuous-depth training.}
The use of adjoints for neural ordinary differential equations is standard \cite{chen2018neuralode}. The present contribution is not the adjoint method alone, but its combination with a causal Volterra derivative, a normalized positional measure, and a decomposition that distinguishes residual transmission from nonlocal causal amplification. The continuous-depth viewpoint is related to control formulations of deep networks \cite{ehanli2019mfoc,huan2026firstorder}; unlike a mean-field state equation, however, the vector field studied here has no independent law-valued coefficient.

\paragraph{Jacobian, stability, and observability regularization.}
Jacobian penalties have been used to improve robustness or stabilize implicit-depth models \cite{bai2021jacobian,hoffman2019jacobian}. Our observability construction is closer to a task-conditioned finite-horizon Gramian: it asks whether perturbations at different token positions remain visible through a fixed family of observation maps. Equalization is therefore relative to a specified task readout, not an intrinsic property of the network.

\paragraph{Long-context position bias and attention sinks.}
Lost-in-the-Middle behavior, positional-attention calibration, and attention sinks have been documented empirically \cite{hsieh2024found,liu2024lost,xiao2024streaming}. Chowdhury's exact birth-time theory \cite{chowdhury2026birth} derives a closed-form initialized U-shape under a Ces\`{a}ro causal-decoder abstraction. Our framework is complementary: it tracks a training-time adjoint density for a controlled continuous-depth model and gives conditional, not universal, boundary conclusions.

\paragraph{Continuum and operator limits of attention.}
Function-space formulations of attention establish continuum operators and discretization consistency in neural-operator settings \cite{calvello2024continuum}. Our Volterra limit specializes this perspective to normalized causal attention on a one-dimensional context interval and couples it to residual-depth and adjoint dynamics.

\paragraph{Novelty statement.}
Relative to the preceding strands of work, the contribution is an integrated forward--backward framework for analyzing positional sensitivity, rather than an unconditional theory asserting that causal Transformers must exhibit a particular boundary shape. On the analytic side, we establish well-posedness and discretization estimates for a continuous-depth causal residual flow with a Volterra attention operator, and couple this framework to an adjoint-based positional influence measure whose normalized evolution along deterministic gradient flow is explicit. On the mechanism side, we separate residual transmission, nonlocal causal-cone amplification, and local channels; retain their cross-covariances in exact regional energy identities; and derive independently checkable sufficient conditions for primacy, persistence of a right-boundary bias, and a middle deficit. On the methodological side, we formulate finite-token diagnostics and candidate regularizers---including influence balancing, positional reweighting, singular-channel tests, and task-conditioned observability balancing---while stating their differentiation requirements, estimator limitations, computational costs, and present validation scope. Accordingly, the approximation and sensitivity identities are unconditional under their stated regularity hypotheses, whereas the conclusions about primacy, recency, and a U-shaped Lost-in-the-Middle profile remain conditional.

\subsection{Mathematical viewpoint}

The framework has three components.
\begin{enumerate}[label=(\roman*)]
\item A causal residual Transformer is treated as an Euler discretization of a continuous-depth dynamical system. The causal mask is retained in the continuum limit as a Volterra support condition on the attention kernel.
\item Full-batch gradient descent is studied through its continuous-time gradient-flow limit. Thus training time is a deterministic variable $s$, and the parameter state follows a single trajectory rather than a stochastic law.
\item The induced distribution of positional influence is defined directly by the adjoint sensitivity density \eqref{eq:influence-density}. Primacy, recency, and lost-in-the-middle are then diagnosed from this density along the gradient-flow path.
\end{enumerate}

All claims are formulated at the level of deterministic continuous training, adjoint calculus, and structural properties of causal residual dynamics. The central message is simple: lost-in-the-middle is a positional sensitivity imbalance; the most direct theoretical cure is to control the training objective or architecture so that the normalized influence density matches a desired target, typically the uniform density.
Equivalently, the positional bias studied here is not treated as a static summary of attention weights. The relevant object is the normalized adjoint-energy density, and deterministic gradient flow changes this density by increasing the relative influence of positions whose adjoint energy decays more slowly than the positional average. Thus training determines which positions gain or lose sensitivity, while the residual identity path supplies a separate structural channel through which localized terminal or right-readout sensitivity can persist backward through depth and appear as recency.

\subsection{Contributions}

The contributions are separated into proved approximation/sensitivity results, conditional mechanism statements, and proposed diagnostics.
\begin{enumerate}[label=(\alph*),leftmargin=2em]
\item \textbf{Analytic well-posedness and discretization.} We formulate the causal residual Transformer as a controlled depth flow on $L^2$, prove regularity of the masked attention--FNN vector field on bounded trajectories, quantify residual-to-ODE shadowing for convergent controls, and establish finite-token-to-Volterra and readout-adjoint consistency estimates.
\item \textbf{Exact adjoint and influence identities.} We define the positional influence measure from input-adjoint energy, derive its exact normalized gradient-flow evolution, and give an exact regional residual/cone/local decomposition including all cross terms.
\item \textbf{Conditional sufficient conditions for boundary advantages.} We prove a cone-channel primacy criterion, a pathwise/moment-based residual-persistence estimate, and an independently checkable double-deficit theorem based on channel lower bounds, middle and local-channel upper bounds, and Cauchy--Schwarz correlation constants.
\item \textbf{Finite-token diagnostics and proposed regularizers.} We formulate influence balancing, positional reweighting, singular-channel diagnostics, and task-aligned observability balancing, and state their second-order differentiation requirements, estimator bias/variance, and computational complexity.
\end{enumerate}

\begin{theorem}[Core rigorous and conditional conclusions]\label{thm:main-concise}

Assume Assumptions \ref{ass:bounded} and \ref{ass:depth-profile}, together with the spatial regularity and differentiability hypotheses stated in the corresponding results below. 
\begin{enumerate}[label=(\roman*),leftmargin=2em]
\item If the interpolated layer controls converge to a limiting control in $L^1(0,T)$, the finite-depth residual recursion, including the split-layer defect, converges to the corresponding depth flow with the error stated in Theorem \ref{thm:resnet-ode}.
\item On bounded $C^1$ position fields and away from the left endpoint, finite-token masked attention converges to the continuum Volterra operator at order $O(L^{-1})$ as stated in Theorem \ref{thm:attention-continuum-error}; the endpoint is defined through the trace of the averaged operator.
\item The input-adjoint energy defines the influence measure of Definition \ref{def:influence}. Under Proposition \ref{prop:influence-differentiability}, its $L^1$ density is differentiable along gradient flow and satisfies \eqref{eq:gd-mdot} in $L^1(0,1)$.
\item The Volterra identity \eqref{eq:volterra-adjoint} supplies a possible left-boundary amplification channel. Under the joint random-kernel, cross-depth coercivity and middle upper-bound conditions of Proposition \ref{prop:primacy-noncancellation}, the depth-integrated cone channel has a quantitative left advantage. The residual identity formula \eqref{eq:continuous-adjoint-duhamel} preserves any sufficiently strong right-localized terminal sensitivity, and Proposition \ref{prop:regional-recency-persistence} quantifies in regional squared energy when that right bias survives to depth zero.
\item Lost-in-the-middle at scale $\delta$ is equivalent to positivity of the lost-in-the-middle index \eqref{eq:gd-lim-index}. Theorem \ref{thm:channel-energy-double-deficit} gives sufficient channel and cross-covariance conditions for this inequality. Under the extra $L^2$-valued differentiability hypothesis of Proposition \ref{prop:influence-l2-differentiability}, the regularizers in \eqref{eq:gd-target-penalty}, \eqref{eq:weighted-cross-entropy}, and \eqref{eq:gd-observability-penalty} are proposed mechanisms for reducing the measured imbalance; they are not asserted to solve it without the additional optimization conditions stated in the algorithmic discussion.
\end{enumerate}
\end{theorem}

\paragraph{Organization of the paper.}
Section~\ref{sec:continuous-depth} gives the model, parameter topology, well-posedness, joint discretization results, and adjoint influence definition. Section~\ref{sec:gd-flow} derives influence-density dynamics, the exact channel decomposition, independently verifiable sufficient conditions and countervailing terms, and the proposed algorithms with complexity estimates. Section~\ref{sec:simulations} reports controlled finite-token experiments. Section~\ref{sec:conclusion} separates proved statements, conditional conclusions, and open problems.

\section{Continuous-Depth Residual Transformers}\label{sec:continuous-depth}
This section fixes the notation and turns a finite-depth causal residual Transformer into a continuous-depth controlled flow. The construction is intentionally explicit about the causal mask, since the triangular support of masked attention is one architectural source of boundary-sensitive positional dynamics.
\subsection{Notation convention}

At finite sequence length, we use the one-based convention
\begin{align*}
X=(x_1,\ldots,x_L)^\top\in\R^{L\times d_x},
\end{align*}
where $d_x$ is the hidden dimension and $L$ is the context length. The vocabulary dimension is $d_v$, the token label at position $l\in\{1,\ldots,L\}$ is denoted by $Y_l(X_0)\in\{e_1,\ldots,e_{d_v}\}$, and $H_L(X)\in\R^{L\times d_v}$ denotes the discrete logit readout. We set $p_l=l/L$ and use the right-endpoint cells $I_l=((l-1)/L,l/L]$. The token-averaged cross-entropy loss is denoted by $\mathcal \ell(Y,H_L(X))$; see \eqref{eq:discrete-cross-entropy} and \eqref{eq:continuum-cross-entropy}. The subscript in the input state $X_0$ denotes depth time and is unrelated to token indexing.

To avoid overloading notation, we distinguish the following variables throughout the paper:
\begin{itemize}[leftmargin=2em]
\item $t\in[0,T]$ is \emph{depth time}. A finite-depth Transformer recursion with step size $\varepsilon=T/M$ is regarded as an explicit Euler discretization of a controlled flow.
\item $\theta_t\in\Theta$ is the depth-dependent layer-parameter control, and $\zeta(t)$ is a scalar depth profile as in \cite{huan2026firstorder}. For a fixed control path, $\rho_t^\theta=\Law(X_t^\theta)$ denotes the law of hidden states along depth.
\item $s\ge0$ is \emph{training time}, the continuous limit of full-batch gradient descent iterations. The deterministic parameter state is denoted by $\vartheta_s$ and evolves according to the gradient flow \eqref{eq:gd-flow}.
\item $p\in[0,1]$ is normalized context position. The main object of this paper is the positional influence density $m_s^{\rm GD}(p)$ along gradient flow, which is distinct from the hidden-state law $\rho_t^\theta$.
\end{itemize}

\subsection{Context continuum and hidden states}

Let $\Omega=(0,1)$ be the normalized context interval. The base state space for the depth evolution is
\[
\mathcal X:=L^2(\Omega;\R^{d_x}),
\]
while sampling and pointwise quadrature statements are made on the regular class
\[
\mathcal X_{\rm reg}:=H^1(\Omega;\R^{d_x})\cap L^\infty(\Omega;\R^{d_x}).
\]
In one spatial dimension, $H^1$ functions have continuous representatives, so $X_t(p)$ is meaningful for $X_t\in\mathcal X_{\rm reg}$. The continuous-depth well-posedness results are formulated in $\mathcal X$; whenever a theorem uses pointwise samples or Lipschitz quadrature, the additional spatial regularity is stated explicitly. A ``perturbation at position $p$'' is interpreted through perturbations supported on shrinking cells around $p$, not through a Dirac mass, which is not an element of $L^2$.

Following the explicit-Euler notation of \cite{huan2026firstorder}, a residual block with $M$ layers, depth step size $\varepsilon=T/M$, and depth profile $\zeta$ has the form
\begin{equation}\label{eq:discrete-residual}
X_{k+1}^{\varepsilon}(p)=X_k^{\varepsilon}(p)+\varepsilon\,\zeta(k\varepsilon)\mathscr F_{\theta_k}\bigl(X_k^{\varepsilon}\bigr)(p),
\qquad k=0,\ldots,M-1.
\end{equation}
Here $\theta_k$ denotes the trainable layer parameter and $\mathscr F_{\theta_k}$ is the complete masked-attention plus tokenwise feed-forward vector field defined in \eqref{eq:complete-vector-field-finite}.  No independent empirical-law or mean-field coefficient is introduced in the modeled dynamics.

For clarity, we now spell out the masked Transformer block that defines the induced vector field $\mathscr F_\theta$.
At finite sequence length, let \(\mathrm{LN}\) denote the regularized smooth layer-normalization map defined tokenwise as follows: for \(z\in\R^{d_x}\), set
\begin{equation}\label{eq:smooth-layernorm}
\bar z:=\frac1{d_x}\sum_{m=1}^{d_x}z_m,
\qquad
\mathrm{LN}(z)
:=\gamma_{\rm LN}\odot
\frac{z-\bar z\mathbf 1}
{\left(d_x^{-1}\sum_{m=1}^{d_x}(z_m-\bar z)^2+\eta_{\rm LN}\right)^{1/2}}
+\beta_{\rm LN},
\qquad \eta_{\rm LN}>0,
\end{equation}
where \(\gamma_{\rm LN},\beta_{\rm LN}\in\R^{d_x}\) are bounded affine parameters and \(\eta_{\rm LN}\) removes the variance singularity. For a sequence \(X\), \(\mathrm{LN}(X)_i\) means \(\mathrm{LN}(X_i)\). Let $N_h$ be the number of attention heads, and for head $h\in\{1,\ldots,N_h\}$, write
\begin{align*}
q_i^h=W_{Q,h}\mathrm{LN}(X)_i,\qquad
k_j^h=W_{K,h}\mathrm{LN}(X)_j,\qquad
v_j^h=W_{V,h}\mathrm{LN}(X)_j.
\end{align*}
The causal mask enters through the lower-triangular attention weights
\begin{equation}\label{eq:causal-attention-weights}
A_{ij}^h(X)
=\frac{\one_{\{1\le j\le i\}}\exp\!\left(\langle q_i^h,k_j^h\rangle/\sqrt{d_k}\right)}
{\sum_{r=1}^{i}\exp\!\left(\langle q_i^h,k_r^h\rangle/\sqrt{d_k}\right)},
\qquad 1\le i,j\le L,
\end{equation}
where $d_k$ is the head dimension. The masked multi-head attention and feed-forward maps are
\begin{align}
\mathrm{Attn}_{\theta}^{\rm c}(X)_i
&=W_O\bigl[\sum_{j=1}^{i}A_{ij}^1(X)v_j^1,\ldots,
\sum_{j=1}^{i}A_{ij}^{N_h}(X)v_j^{N_h}\bigr],\label{eq:masked-mha}\\
\mathrm{FNN}_{\theta}(Z)_i
&=W_2\,\sigma\bigl(W_1\mathrm{LN}(Z)_i+b_1\bigr)+b_2, \qquad 1\le i\le L.\label{eq:mlp-block}
\end{align}
A pre-normalized residual Transformer layer may therefore be written as the split Euler step
\begin{equation}\label{eq:complete-transformer-layer}
\widetilde X_{k+1}^{\varepsilon}=X_k^{\varepsilon}
+\varepsilon\zeta(k\varepsilon)\mathrm{Attn}_{\theta_k}^{\rm c}(X_k^{\varepsilon}),
\qquad
X_{k+1}^{\varepsilon}=\widetilde X_{k+1}^{\varepsilon}
+\varepsilon\zeta(k\varepsilon)\mathrm{FNN}_{\theta_k}(\widetilde X_{k+1}^{\varepsilon}).
\end{equation}
Equivalently, up to an $O(\varepsilon^2)$ splitting error, \eqref{eq:complete-transformer-layer} has the one-field form \eqref{eq:discrete-residual} with
\begin{equation}\label{eq:complete-vector-field-finite}
\mathscr F_{\theta_k}(X_k^{\varepsilon})
=\mathrm{Attn}_{\theta_k}^{\rm c}(X_k^{\varepsilon})+
\mathrm{FNN}_{\theta_k}(X_k^{\varepsilon}).
\end{equation}
In the continuum-position notation, the corresponding masked attention component has Volterra form
\begin{align}
\mathrm{Attn}_{\theta}^{\rm c}(X)(p)
&=W_O\left[\int_{0}^{p}A_\theta^1(p,q;X)W_{V,1}\mathrm{LN}(X(q))\dd q,\ldots,
\int_{0}^{p}A_\theta^{N_h}(p,q;X)W_{V,N_h}\mathrm{LN}(X(q))\dd q\right],\label{eq:continuum-causal-attention}\\
A_\theta^h(p,q;X)
&=\frac{\one_{\{0\le q\le p\}}\exp\!\left(\langle W_{Q,h}\mathrm{LN}(X(p)),W_{K,h}\mathrm{LN}(X(q))\rangle/\sqrt{d_k}\right)}
{\int_{0}^{p}\exp\!\left(\langle W_{Q,h}\mathrm{LN}(X(p)),W_{K,h}\mathrm{LN}(X(r))\rangle/\sqrt{d_k}\right)\dd r},
\qquad p>0.
\label{eq:continuum-causal-kernel}
\end{align}
For $p>0$, it is sometimes convenient to encode the normalized kernel as the headwise prefix-attention probability measure
\begin{equation}\label{eq:attention-induced-measure}
\mu_{\theta,X}^{h,p}(\dd q)
:=A_\theta^h(p,q;X)\dd q,
\qquad
\mu_{\theta,X}^{h,p}([0,p])=1,
\qquad
\operatorname{supp}\mu_{\theta,X}^{h,p}\subset[0,p].
\end{equation}
Thus the notation $\mu_X$ used below, when it is useful as shorthand, refers only to the collection $\{\mu_{\theta,X}^{h,p}\}_{h,p}$ induced by \eqref{eq:continuum-causal-attention}--\eqref{eq:continuum-causal-kernel}; it is not an additional law-valued state variable.
The kernel itself is not assigned a finite value at $p=0$: even for a constant score it equals $p^{-1}\one_{[0,p]}(q)$. What has a trace is the averaged attention operator. For $X\in\mathcal X_{\rm reg}$ we define
\begin{equation}\label{eq:continuum-attention-left-trace}
\mathrm{Attn}_{\theta}^{\rm c}(X)(0)
:=W_O\bigl[W_{V,1}\mathrm{LN}(X(0)),\ldots,
W_{V,N_h}\mathrm{LN}(X(0))\bigr],
\end{equation}
which is the limit of \eqref{eq:continuum-causal-attention} as $p\downarrow0$. Thus the first-token rule is represented by the trace of the operator, not by the pointwise limit of the normalized kernel. Globally in $L^2$, the model operator has the same endpoint structure as the Hardy averaging operator $p^{-1}\int_0^p f(q)\dd q$, whose boundedness is used in the regularity argument.

The essential point is the support condition $A_\theta^h(p,q;X)=0$ for $q>p$. Causal masking therefore makes the depth dynamics triangular in the position variable: position $p$ can aggregate only positions $q\le p$. This triangular structure is one of the architectural sources that can create boundary wells in the effective positional potential $V$ and links the present training-time influence theory to the birth-time position-bias mechanism of \cite{chowdhury2026birth}. Throughout the regularity, depth-flow, adjoint, and primacy analyses below, $\mathscr F_\theta$ denotes only the causal masked-attention plus tokenwise feed-forward block in \eqref{eq:complete-vector-field-finite}--\eqref{eq:continuum-causal-kernel}.  Equivalently, the only retained $X$-dependent measure is the prefix-attention family \eqref{eq:attention-induced-measure}.  Every quantity at output position $p$ therefore depends only on $X|_{[0,p]}$; no independent law of $X$, global empirical statistic, or other noncausal lifting is included.

All matrix, bias, layer-normalization, and readout parameters are collected in a finite-dimensional product space
\[
\Theta\cong\R^{d_\theta},
\]
equipped with the Euclidean product norm (equivalently, with the Frobenius norm on matrix blocks and the Euclidean norm on vector blocks). Control paths are measured in the Bochner norms $L^1(0,T;\Theta)$ and $L^\infty(0,T;\Theta)$.

\begin{assumption}[Parameter topology and boundedness]\label{ass:bounded}
The admissible set $\Theta_{\rm ad}\subset\Theta$ is closed and bounded, hence compact since $\Theta$ is finite dimensional. All uniform parameter constants below depend only on explicit bounds for the blocks in $\Theta_{\rm ad}$. The regularized layer-normalization map \eqref{eq:smooth-layernorm} and the activation $\sigma$ have bounded first derivatives on bounded sets.
\end{assumption}

\begin{assumption}[Depth profile and depth-control regularity]\label{ass:depth-profile}
The scalar depth profile is nonnegative and satisfies
\begin{align*}
0\le \zeta(t)\le \|\zeta\|_{L^\infty(0,T)}<\infty
\quad\text{for a.e. }t,
\qquad
\mathrm{Lip}(\zeta)<\infty.
\end{align*}
The depth-control path \(t\mapsto\theta_t\) is Lipschitz, or piecewise Lipschitz with finitely many pieces and uniform Lipschitz constants, with values in a bounded admissible parameter set \(\Theta_{\rm ad}\).  
\end{assumption}

\begin{remark}
The assumption for the limiting depth-control path \(t\mapsto\theta_t\) is a continuous-depth idealization of the finite-layer situation.  An $M$-layer network defines a layerwise interpolation $\theta^\varepsilon$, but bounded weights alone do not force these interpolants to converge to one common control as $M\to\infty$.  Theorem \ref{thm:resnet-ode} therefore distinguishes two statements: an $O(\varepsilon)$ Euler shadowing estimate for each interpolated control, and convergence to a single continuous-depth flow only when $\theta^\varepsilon\to\theta$ in $L^1(0,T)$.  Smooth or bounded-variation depth parameterizations are standard sufficient ways to obtain this convergence.
\end{remark}

\begin{proposition}[Regularity of the masked Transformer vector field]\label{prop:transformer-lipschitz}
Assume that Assumptions \ref{ass:bounded} and \ref{ass:depth-profile} hold.  Let
\begin{equation}\label{eq:induced-transformer-field}
\mathscr F_\theta(X)
:=\mathrm{Attn}_{\theta}^{\rm c}(X)+\mathrm{FNN}_{\theta}(X),
\end{equation}
where the attention term is given by \eqref{eq:continuum-causal-attention}--\eqref{eq:continuum-causal-kernel}, or equivalently by integration against the prefix-attention measures \eqref{eq:attention-induced-measure}.  Assume also that the matrix and affine parameterization of $\mathscr F_\theta$ is locally Lipschitz in $\theta$, uniformly on bounded state sets.
Then $\mathscr F_\theta(X)\in\mathcal X$, and the induced field has the following properties on bounded subsets of $\mathcal X$.  For every $R>0$ there exists $C_{F,R}>0$ such that, for all admissible $\theta\in\Theta_{\rm ad}$ and all $X,Y\in\mathcal X$ with $\|X\|_{\mathcal X},\|Y\|_{\mathcal X}\le R$,
\begin{equation}\label{eq:transformer-field-state-lipschitz}
\|\mathscr F_\theta(X)-\mathscr F_\theta(Y)\|_{\mathcal X}
\le C_{F,R}\|X-Y\|_{\mathcal X}.
\end{equation}
Moreover, $\mathscr F_\theta(X)$ is uniformly bounded on bounded subsets of $\mathcal X$, uniformly over admissible controls: for every $R>0$ there exists $M_R<\infty$ such that
\begin{equation}\label{eq:transformer-field-uniform-bound}
\sup_{\theta\in\Theta_{\rm ad}}\sup_{\|X\|_{\mathcal X}\le R}
\|\mathscr F_\theta(X)\|_{\mathcal X}\le M_R.
\end{equation}
In addition, with $G(t,X):=\zeta(t)\mathscr F_{\theta_t}(X)$ and $t_k=k\varepsilon$, there exists $C_G>0$ such that, for $X$ in a bounded set and whenever $t$ and $t_k$ lie in the same Lipschitz piece of the control path,
\begin{equation}\label{eq:depth-time-consistency}
\|G(t,X)-G(t_k,X)\|_{\mathcal X}\le C_G|t-t_k|.
\end{equation}
\end{proposition}

\begin{remark}[Bounded state sets under pre-layer normalization]\label{rem:bounded-state-preln}
The bounded-state restriction in Proposition \ref{prop:transformer-lipschitz} is compatible with the pre-layer-normalized architecture in \eqref{eq:complete-transformer-layer}.  In each attention and feed-forward sublayer, the trainable maps act on $\mathrm{LN}(X_i)$, not directly on the unnormalized token state.  By \eqref{eq:smooth-layernorm}, for bounded affine layer-normalization parameters one has a uniform bound
\begin{equation}\label{eq:preln-uniform-normalized-bound}
\|\mathrm{LN}(z)\|\le C_{\rm LN},
\qquad z\in\R^{d_x},
\end{equation}
where the constant depends only on $d_x$, $\eta_{\rm LN}$, and the bounds on $\gamma_{\rm LN}$ and $\beta_{\rm LN}$.  Hence the attention values, queries, keys, and feed-forward inputs are evaluated on a bounded normalized set, uniformly over the unnormalized residual stream.  The residual update still adds increments to $X$, but on a finite depth horizon these increments are bounded by \eqref{eq:transformer-field-uniform-bound} and Gronwall's inequality.  Thus, if the input set is bounded in $\mathcal X$, the entire pre-layer-normalized residual trajectory remains in a bounded state set depending only on the input radius, the admissible parameter bounds, and $T$.
\end{remark}

\begin{remark}[Scope of the attention-induced measure]\label{rem:attention-measure-scope}
The family $\mu_{\theta,X}^{h,p}$ in \eqref{eq:attention-induced-measure} is only a notation for the normalized causal attention kernel.  It is neither the probability law $\mathcal L(X(U))$ of a randomly sampled token state nor an independent measure-valued argument of the vector field.  Consequently, the derivative $D\mathscr F_\theta(X)$ is obtained by differentiating the explicit normalized attention formula and the tokenwise feed-forward map; no separate Lipschitz assumption in a measure metric and no Lions derivative are required.  A genuinely global statistic, such as an integral over the full context that enters every output position, is excluded from the present model since its derivative would add a non-Volterra channel and would invalidate the purely lower-triangular causal decomposition used in the primacy analysis.
\end{remark}

\begin{proof}
We give the endpoint estimate explicitly.  Fix one head and abbreviate
\[
u_X(p):=\mathrm{LN}(X(p)),\qquad
V_X(p):=W_Vu_X(p),\qquad
h(p):=\|X(p)-Y(p)\|.
\]
Regularization by $\eta_{\rm LN}>0$ and bounded affine parameters make $u_X$ globally bounded and globally Lipschitz.  Hence the scores $s_X(p,q)$ are uniformly bounded, say $|s_X|\le B$, and
\begin{equation}\label{eq:score-l2-difference}
|s_X(p,q)-s_Y(p,q)|\le C\bigl(h(p)+h(q)\bigr).
\end{equation}
Write
\[
a_X(p,q):=\frac{e^{s_X(p,q)}}{Z_X(p)},\qquad
Z_X(p):=\int_0^p e^{s_X(p,r)}\dd r.
\]
For $p>0$,
\begin{equation}\label{eq:attention-density-hardy-bound}
pe^{-B}\le Z_X(p)\le pe^B,\qquad
0\le a_X(p,q)\le \frac{e^{2B}}p.
\end{equation}
Using the quotient identity for $a_X-a_Y$, the Lipschitz continuity of the exponential on $[-B,B]$, and \eqref{eq:score-l2-difference}, we obtain
\begin{equation}\label{eq:normalized-kernel-l1-difference}
\int_0^p|a_X(p,q)-a_Y(p,q)|\dd q
\le C\left(h(p)+\frac1p\int_0^ph(q)\dd q\right).
\end{equation}
Let $(Hh)(p):=p^{-1}\int_0^ph(q)\dd q$ be the Hardy averaging operator.  For the head output $\mathcal T(X)(p)=\int_0^pa_X(p,q)V_X(q)\dd q$, add and subtract $\int_0^pa_X(p,q)V_Y(q)\dd q$.  The global bound on $V_Y$, the Lipschitz bound on $V_X-V_Y$, and \eqref{eq:attention-density-hardy-bound}--\eqref{eq:normalized-kernel-l1-difference} give
\begin{equation}\label{eq:attention-pointwise-hardy-stability}
\|\mathcal T(X)(p)-\mathcal T(Y)(p)\|
\le C\bigl(h(p)+(Hh)(p)\bigr).
\end{equation}
Hardy's inequality on $(0,1)$,
\[
\|Hh\|_{L^2(0,1)}\le2\|h\|_{L^2(0,1)},
\]
therefore yields
\begin{equation}\label{eq:global-l2-attention-stability}
\|\mathcal T(X)-\mathcal T(Y)\|_{L^2}
\le C\|X-Y\|_{\mathcal X}.
\end{equation}
This controls the whole neighborhood of $p=0$; the trace \eqref{eq:continuum-attention-left-trace} is needed only to identify the continuous representative at the single endpoint, not to prove the $L^2$ bound.

Concatenating finitely many heads and applying $W_O$ preserves \eqref{eq:global-l2-attention-stability}.  The feed-forward map is pointwise Lipschitz since its layer-normalized input is bounded.  Thus the explicit dependence of the attention measures $\mu_{\theta,X}^{h,p}$ on $X$, already controlled by \eqref{eq:normalized-kernel-l1-difference}, gives \eqref{eq:transformer-field-state-lipschitz} directly; no additional law-valued coefficient is present.  Normalization of the attention weights and boundedness of the normalized values and feed-forward inputs give \eqref{eq:transformer-field-uniform-bound}.

Finally, local Lipschitz dependence on $\theta$, Lipschitz continuity of $\zeta$, and Assumption \ref{ass:depth-profile} imply, for $t,r$ in the same Lipschitz control piece,
\[
\|\zeta(t)\mathscr F_{\theta_t}(X)-\zeta(r)\mathscr F_{\theta_r}(X)\|_{\mathcal X}
\le C_R|t-r|,
\]
which proves \eqref{eq:depth-time-consistency} on each such piece.
\end{proof}

The previous proposition justifies treating the discrete masked Transformer block as a regular vector field on bounded state sets. We now record the complementary consistency estimate showing that the finite-token causal attention formula is approximated by the continuum Volterra attention operator away from the singular left endpoint.

\begin{theorem}[Finite-token attention to Volterra attention]\label{thm:attention-continuum-error}
Let \(p_i=i/L\), \(1\le i\le L\), and let \(X_L=(X(p_1),\ldots,X(p_L))\) be samples of a function \(X\in C^1([0,1];\R^{d_x})\). Assume that \(X\in C^1([0,1];\R^{d_x})\) with \(\|X\|_{C^1}\le R\) and Assumption \ref{ass:bounded} holds. Then, for every \(\eta\in(0,1)\), there exists \(C_{R,\eta}<\infty\)\footnote{The subscript \(R\) indicates that the constant is uniform for states in the bounded set \(\|X\|_{C^1}\le R\) and for admissible parameters in \(\Theta_{\rm ad}\), while the subscript \(\eta\) records the lower bound \(p_i\ge\eta\) away from the left endpoint.} independent of \(L\), such that for every \(i\) with \(p_i\in[\eta,1]\),
\begin{equation}\label{eq:attention-continuum-pointwise-error}
\left\|
\mathrm{Attn}_{\theta}^{\rm c}(X_L)_i
-
\mathrm{Attn}_{\theta}^{\rm c}(X)(p_i)
\right\|
\le \frac{C_{R,\eta}}{L} .
\end{equation}
Consequently, in the discrete normalized \(\ell^2\) norm away from the left boundary,
\begin{equation}\label{eq:attention-continuum-l2-error}
\left(
\frac1L\sum_{i:\,p_i\in[\eta,1]}
\left\|
\mathrm{Attn}_{\theta}^{\rm c}(X_L)_i
-
\mathrm{Attn}_{\theta}^{\rm c}(X)(p_i)
\right\|^2
\right)^{1/2}
\le \frac{C_{R,\eta}}{L} .
\end{equation}
At the left boundary, the comparison is understood through the operator trace \eqref{eq:continuum-attention-left-trace}, which agrees with the first-token finite attention rule.
Moreover, there is a constant $C_R$, independent of $L$ and $i$, such that the endpoint-aware estimate
\begin{equation}\label{eq:attention-global-cell-error}
\left\|
\mathrm{Attn}_{\theta}^{\rm c}(X_L)_i-
\mathrm{Attn}_{\theta}^{\rm c}(X)(p_i)
\right\|
\le \frac{C_R}{i},\qquad 1\le i\le L,
\end{equation}
holds. Consequently, if $n_L=\lfloor\delta L\rfloor\ge1$, then the average error over the first $n_L$ cells satisfies
\begin{equation}\label{eq:attention-left-regional-error}
\frac1{n_L}\sum_{i=1}^{n_L}
\left\|
\mathrm{Attn}_{\theta}^{\rm c}(X_L)_i-
\mathrm{Attn}_{\theta}^{\rm c}(X)(p_i)
\right\|
\le \frac{C_R(1+\log n_L)}{n_L}.
\end{equation}
Thus regional primacy averages may include the first cells.
\end{theorem}

\begin{proof}
It is enough to prove the estimate head by head, since concatenation over finitely many heads and multiplication by the bounded matrix \(W_O\) only change the constant. Fix one head. Since \(X\in C^1([0,1];\R^{d_x})\) with \(\|X\|_{C^1}\le R\), Assumption \ref{ass:bounded} implies that the regularized layer-normalization map and the linear maps \(W_{Q,h},W_{K,h},W_{V,h}\) are uniformly bounded and uniformly Lipschitz on the range of \(X\), uniformly over \(\theta\in\Theta_{\rm ad}\). Hence the score
\begin{equation}\label{eq:continuum-attention-score}
S_h(p,q;X)
:=\frac{\left\langle W_{Q,h}\mathrm{LN}(X(p)),W_{K,h}\mathrm{LN}(X(q))\right\rangle}{\sqrt{d_k}}
\end{equation}
is uniformly bounded and Lipschitz in \((p,q)\) on \([0,1]^2\). In particular, there are constants \(B_R,L_R<\infty\), depending only on the bounded state set and the admissible parameter set, such that
\begin{equation}\label{eq:score-bounded-lipschitz-from-bounded-assumption}
|S_h(p,q;X)|\le B_R,
\qquad
|S_h(p,q;X)-S_h(p',q';X)|\le L_R(|p-p'|+|q-q'|).
\end{equation}
The map \(q\mapsto W_{V,h}\mathrm{LN}(X(q))\) is also uniformly Lipschitz and bounded. Since the exponential is Lipschitz on \([-B_R,B_R]\), the two integrands
\[
q\mapsto \exp(S_h(p_i,q;X))W_{V,h}\mathrm{LN}(X(q)),
\qquad
q\mapsto \exp(S_h(p_i,q;X))
\]
are uniformly Lipschitz in \(q\), for \(p_i\in[\eta,1]\). Define
\begin{align}
N_i^L&:=\frac1L\sum_{j=1}^i
\exp(S_h(p_i,p_j;X))\,W_{V,h}\mathrm{LN}(X(p_j)),\label{eq:discrete-attention-numerator-scaled}\\
D_i^L&:=\frac1L\sum_{j=1}^i
\exp(S_h(p_i,p_j;X)),\label{eq:discrete-attention-denominator-scaled}\\
N(p_i)&:=\int_0^{p_i}\exp(S_h(p_i,q;X))W_{V,h}\mathrm{LN}(X(q))\dd q,\label{eq:continuum-attention-numerator}\\
D(p_i)&:=\int_0^{p_i}\exp(S_h(p_i,q;X))\dd q.\label{eq:continuum-attention-denominator}
\end{align}
The factor \(1/L\) cancels in the quotient \(N_i^L/D_i^L\), so this quotient is exactly the finite attention average in \eqref{eq:masked-mha}. The right-endpoint Riemann-sum estimate for uniformly Lipschitz integrands gives
\begin{equation}\label{eq:riemann-attention-error}
\|N_i^L-N(p_i)\|+|D_i^L-D(p_i)|\le \frac{C_R}{L},
\end{equation}
with \(C_R\) independent of \(L\), \(i\), and the admissible parameter.
Moreover, the score bound gives direct lower bounds for both denominators, for every $L\ge1$ and every $1\le i\le L$:
\begin{equation}\label{eq:attention-two-denominator-lower-bounds}
D(p_i)\ge p_i e^{-B_R},
\qquad
D_i^L=\frac1L\sum_{j=1}^i e^{S_h(p_i,p_j;X)}\ge p_i e^{-B_R}.
\end{equation}
Applying the quotient identity for the difference of the two normalized numerators, the Riemann bound \eqref{eq:riemann-attention-error}, and $\|N(p_i)\|\le C_Rp_i$ gives directly
\[
\left\|\frac{N_i^L}{D_i^L}-\frac{N(p_i)}{D(p_i)}\right\|
\le \frac{C_R}{Lp_i}=\frac{C_R}{i}.
\]
This proves \eqref{eq:attention-global-cell-error}. If $p_i\ge\eta$, then $i\ge\eta L$ and the same estimate gives \eqref{eq:attention-continuum-pointwise-error}; the normalized $\ell^2$ bound follows by squaring and summing. Finally,
\[
\frac1{n_L}\sum_{i=1}^{n_L}\frac{C_R}{i}
\le \frac{C_R(1+\log n_L)}{n_L},
\]
which proves \eqref{eq:attention-left-regional-error} and explicitly treats the first cells in a left-boundary regional average.
\end{proof}

\subsection{The continuous-depth limit}

The continuous-depth model is the neural ODE generated by the explicit masked Transformer field,
\begin{equation}\label{eq:continuous-depth}
\partial_t X_t(p)=\zeta(t)\mathscr F_{\theta_t}(X_t)(p),
\qquad X_0=x,
\end{equation}
where $t\mapsto\theta_t$ is a measurable control path and the only state-dependent measures inside $\mathscr F_{\theta_t}$ are the prefix-attention measures \eqref{eq:attention-induced-measure}.  In this formulation the network depth is a control horizon, and $\rho_t^\theta=\Law(X_t^\theta)$ denotes only the induced law of hidden states across the data distribution; it is not an argument of the vector field.

\begin{theorem}[Residual-to-ODE limit for convergent controls]\label{thm:resnet-ode}
Assume Assumptions \ref{ass:bounded} and \ref{ass:depth-profile}, and the local Lipschitz dependence on $\theta$ stated in Proposition \ref{prop:transformer-lipschitz}.  Let $\theta\in W^{1,\infty}(0,T;\Theta_{\rm ad})$ be a fixed limiting control.  For $\varepsilon=T/M$, let
\[
\theta^\varepsilon(t)=\theta_k^\varepsilon,
\qquad t\in[t_k,t_{k+1}),
\qquad t_k=k\varepsilon,
\]
be a uniformly bounded layerwise control satisfying
\begin{equation}\label{eq:control-l1-convergence}
\|\theta^\varepsilon-\theta\|_{L^1(0,T)}\longrightarrow0.
\end{equation}
Suppose the finite-depth update has the form
\begin{equation}\label{eq:discrete-residual-with-defect}
X_{k+1}^{\varepsilon}
=X_k^{\varepsilon}
+\varepsilon\zeta(t_k)\mathscr F_{\theta_k^\varepsilon}
  (X_k^{\varepsilon})
+d_k^\varepsilon,
\qquad
\|d_k^\varepsilon\|_{\mathcal X}\le C_{\rm split}\varepsilon^2.
\end{equation}
The one-field recursion \eqref{eq:discrete-residual} has $d_k^\varepsilon=0$, while the split attention--FNN step \eqref{eq:complete-transformer-layer} satisfies \eqref{eq:discrete-residual-with-defect} on bounded trajectories under the stated smoothness assumptions.  Let $X^\theta$ solve \eqref{eq:continuous-depth} with the fixed control $\theta$.  Then there is $C_T$, independent of $\varepsilon$, such that the piecewise constant interpolation satisfies
\begin{equation}\label{eq:resnet-ode-convergent-control-bound}
\sup_{0\le t\le T}
\|X_t^{\varepsilon}-X_t^\theta\|_{\mathcal X}
\le C_T\left(
\varepsilon+
\|\theta^\varepsilon-\theta\|_{L^1(0,T)}
\right).
\end{equation}
In particular, the discrete trajectories converge to the single depth flow $X^\theta$.  If no limiting control is assumed, the same Euler calculation gives only an $O(\varepsilon)$ shadowing estimate relative to the $\varepsilon$-dependent ODE driven by $\theta^\varepsilon$, not convergence to a common flow.
\end{theorem}
The proof is deferred to Appendix \ref{app:proof-resnet-ode}.

\subsection{Adjoint sensitivities and positional influence}
We now define positional influence. For an input-label pair $(X_0,Y)$ and a depth-control path $\theta=(\theta_t)_{0\le t\le T}$, let $X_t^\theta$ be the solution of \eqref{eq:continuous-depth}. The single-example terminal loss is
\begin{equation}\label{eq:single-example-loss}
\mathscr L(\theta;X_0,Y):=\mathcal\ell\bigl(Y,H(X_T^\theta)\bigr),
\end{equation}
where $H:\mathcal X\to L^2(\Omega;\R^{d_v})$ is the logit readout and $\mathcal\ell(Y,H(X))$ is the token-averaged cross-entropy. In finite length, this means
\begin{align}\label{eq:discrete-cross-entropy}
\mathcal\ell(Y,H_L(X))=-\frac1L\sum_{l=1}^{L}Y_l^\top\log\operatorname{softmax}(H_L(X)_l).
\end{align}
The finite-token terminal adjoint is therefore
\begin{equation}\label{eq:terminal-adjoint-softmax}
P_{T,L}^{\theta,X_0,Y}
=D_X\{\mathcal\ell(Y,H_L(X))\}\big|_{X=X_{T,L}^\theta}
=DH_L(X_{T,L}^\theta)^*\Bigl(\frac1L\bigl(\operatorname{softmax}(H_L(X_{T,L}^\theta))-Y\bigr)\Bigr).
\end{equation}
Here the subscript \(L\) records the finite-token discretization rather than the network depth.  More precisely,
\[
H_L:\R^{L\times d_x}\longrightarrow\R^{L\times d_v}
\]
is the discrete readout acting on an \(L\)-token hidden-state matrix, and \(H_L(X)_l\in\R^{d_v}\) is the logit vector at token \(l\).  The state \(X_{T,L}^{\theta}\in\R^{L\times d_x}\) is the terminal hidden matrix obtained after evolving the \(L\)-token residual recursion to depth time \(T\) under the control path \(\theta\); equivalently, if \(T=M\varepsilon\), then \(X_{T,L}^{\theta}=X_{M,L}^{\varepsilon,\theta}\).  Thus \(P_{T,L}^{\theta,X_0,Y}\) is the Euclidean gradient of the token-averaged loss with respect to this terminal finite-token state.  The continuum objects \(H\) and \(X_T^\theta\) below are the corresponding position-field readout and terminal depth-flow state, and Theorem \ref{thm:loss-adjoint-continuum-error} compares the cellwise embedding of the finite-token adjoint with the continuum adjoint.

In the continuum-position normalization, $Y:[0,1]\to\R^{d_v}$ and $H(X):[0,1]\to\R^{d_v}$, and token averaging is replaced by normalized Lebesgue integration:
\begin{align}
\mathcal\ell(Y,H(X))
&=-\int_0^1 Y(p)^\top\log\operatorname{softmax}(H(X)(p))\dd p,\label{eq:continuum-cross-entropy}\\
\left.\frac{\dd}{\dd\epsilon}\mathcal\ell(Y,H(X)+\epsilon U)\right|_{\epsilon=0}
&=\int_0^1\bigl(\operatorname{softmax}(H(X)(p))-Y(p)\bigr)^\top U(p)\dd p.\label{eq:continuum-logit-variation}
\end{align}
Hence, with respect to the $L^2(\Omega;\R^{d_v})$ pairing, the continuum logit-gradient is the pointwise residual $\operatorname{softmax}(H(X)(\cdot))-Y(\cdot)$, and the continuum terminal adjoint is
\begin{equation}\label{eq:continuum-terminal-adjoint}
P_T^{\theta,X_0,Y}
=DH(X_T^\theta)^*\Bigl(\operatorname{softmax}(H(X_T^\theta)(\cdot))-Y(\cdot)\Bigr).
\end{equation}

\begin{theorem}[Discrete-to-continuum loss and terminal adjoint]\label{thm:loss-adjoint-continuum-error}
Let \(p_l=l/L\) for \(l=1,\ldots,L\), and suppose that $Z(p):=H(X)(p)$ and $Y(p)$ are Lipschitz on \([0,1]\), with \(Y_l=Y(p_l)\) and \(H_L(X_L)_l=Z(p_l)\). Assume also that the logit values remain in a bounded set, so that
\[
g(p):=-Y(p)^\top\log\operatorname{softmax}(Z(p))
\]
is Lipschitz with constant \(L_g\). Then
\begin{equation}\label{eq:cross-entropy-discrete-continuum-error}
\left|
-\frac1L\sum_{l=1}^{L}Y_l^\top\log\operatorname{softmax}(H_L(X_L)_l)
+
\int_0^1Y(p)^\top\log\operatorname{softmax}(H(X)(p))\dd p
\right|
\le \frac{L_g}{L} .
\end{equation}
Moreover, define the continuum logit residual
\[
\Delta(p):=\operatorname{softmax}(Z(p))-Y(p)
\]
and its sampled vector \(\Delta_L=(\Delta(p_1),\ldots,\Delta(p_L))\). Assume that the continuum readout adjoint is bounded on the relevant state set,
\begin{equation}\label{eq:readout-adjoint-boundedness}
\|DH(X)^*U\|_{\mathcal X}\le M_R\|U\|_{L^2(0,1;\R^{d_v})},\qquad \|X\|\le R,
\end{equation}
and that the discrete readout adjoint is first-order consistent after the finite-token gradient is interpreted as a cellwise density.\footnote{That is, if \(I_l=((l-1)/L,l/L]\), define the cellwise density embedding by
\begin{equation}\label{eq:terminal-adjoint-density-embedding}
\mathcal E_L\!\bigl(LP_{T,L}^{\theta,X_0,Y}\bigr)(p)
:=\bigl[DH_L(X_L)^*\Delta_L\bigr]_l,
\qquad p\in I_l;
\end{equation}suppose that, for every Lipschitz residual field \(\Delta\),
\begin{equation}\label{eq:readout-adjoint-consistency}
\left\|\mathcal E_L\!\bigl(DH_L(X_L)^*\Delta_L\bigr)-DH(X)^*\Delta\right\|_{\mathcal X}\le \frac{C_R}{L}.
\end{equation}}
Then there exists \(C_R<\infty\), independent of \(L\), such that
\begin{equation}\label{eq:terminal-adjoint-discrete-continuum-error}
\left\|\mathcal E_L\!\bigl(LP_{T,L}^{\theta,X_0,Y}\bigr)-P_T^{\theta,X_0,Y}\right\|_{\mathcal X}\le \frac{C_R}{L},
\end{equation}
where \(P_T^{\theta,X_0,Y}\) is the continuum terminal adjoint in \eqref{eq:continuum-terminal-adjoint}.
\end{theorem}

\begin{remark}
In practice this Lipschitz sampling assumption is the continuum version of the usual finite-token regularity imposed before passing from a sequence to a position field. The hidden path \(p\mapsto X(p)\) is taken from a bounded smooth or piecewise smooth interpolation of the discrete hidden states, and the readout is local or uniformly Lipschitz on bounded sets. Hence
\[
\|Z(p)-Z(q)\|=\|H(X)(p)-H(X)(q)\|\le L_H\|X(p)-X(q)\|\le L_H L_X |p-q|.
\]
For labels, the statement is exact when the continuum label field is obtained by interpolation or piecewise smoothing of the token labels and then sampled at \(p_l\). For discontinuous one-hot labels, the same estimate can be read after replacing the Lipschitz assumption by bounded variation or by applying the argument on intervals between finitely many label jumps; the resulting Riemann error remains first order away from those jumps, with constants depending on the total variation. Thus the hypothesis is not an additional modeling mechanism, but the regularity needed to compare the token average in \eqref{eq:discrete-cross-entropy} with the Lebesgue integral in \eqref{eq:continuum-cross-entropy}.
\end{remark}

\begin{remark}[Linear readout consistency]\label{rem:linear-readout-consistency}
For the common local linear readout
\begin{equation}\label{eq:linear-readout-example}
H(X)(p)=X(p)W_{\rm out},
\end{equation}
where \(W_{\rm out}\in\R^{d_x\times d_v}\) is bounded, the consistency assumption in \eqref{eq:readout-adjoint-consistency} is immediate.  The continuum readout derivative is
\begin{equation}\label{eq:linear-readout-continuum-adjoint}
DH(X)U(p)=U(p)W_{\rm out},
\qquad
DH(X)^*\Delta(p)=\Delta(p)W_{\rm out}^{\!*},
\end{equation}
where \(W_{\rm out}^{\!*}\) denotes the Euclidean adjoint of \(W_{\rm out}\).  At finite length, if \(H_L(X_L)_l=X_lW_{\rm out}\), then
\begin{equation}\label{eq:linear-readout-discrete-adjoint}
\bigl[DH_L(X_L)^*\Delta_L\bigr]_l=\Delta(p_l)W_{\rm out}^{\!*}.
\end{equation}
Consequently the cellwise density embedding satisfies
\begin{equation}\label{eq:linear-readout-cellwise-consistency}
\mathcal E_L\!\bigl(DH_L(X_L)^*\Delta_L\bigr)(p)
=\Delta(p_l)W_{\rm out}^{\!*},
\qquad p\in I_l.
\end{equation}
If \(\Delta\) is Lipschitz, then
\begin{align}
\left\|
\mathcal E_L\!\bigl(DH_L(X_L)^*\Delta_L\bigr)-DH(X)^*\Delta
\right\|_{\mathcal X}^2
&=\sum_{l=1}^{L}\int_{(l-1)/L}^{p_l}
\|\bigl(\Delta(p_l)-\Delta(p)\bigr)W_{\rm out}^{\!*}\|^2\dd p\notag\\
&\le \|W_{\rm out}\|_{\rm op}^2\,\mathrm{Lip}(\Delta)^2
\sum_{l=1}^{L}\int_{(l-1)/L}^{p_l}|p-p_l|^2\dd p\notag\\
&\le \frac{\|W_{\rm out}\|_{\rm op}^2\mathrm{Lip}(\Delta)^2}{L^2}.\label{eq:linear-readout-first-order-bound}
\end{align}
Taking the square root gives the first-order bound required in \eqref{eq:readout-adjoint-consistency}.  Thus, for a tokenwise linear readout, the only discrepancy between the finite-token adjoint density and the continuum adjoint is the standard cellwise sampling error of the logit residual.

The same conclusion holds for tokenwise $C^1$ nonlinear readouts with uniformly Lipschitz derivatives on the bounded state set, and for uniformly banded local readouts whose discrete kernels are consistent quadratures of a local continuum kernel. A genuinely nonlocal readout is permitted in the loss and adjoint definitions, but it can spread terminal sensitivity across the whole context. It therefore preserves the causal interpretation of the \emph{forward} masked dynamics while invalidating any recency argument that relies on a localized terminal adjoint. Such a readout must be analyzed through its actual adjoint kernel rather than through the local-readout examples used below.

\end{remark}

\begin{proof}
The loss estimate is the right-endpoint Riemann-sum error for the Lipschitz function \(g\):
\[
\left|\frac1L\sum_{l=1}^{L}g(p_l)-\int_0^1 g(p)\dd p\right|
\le
\sum_{l=1}^{L}\int_{(l-1)/L}^{p_l}|g(p_l)-g(p)|\dd p
\le
\frac{L_g}{L}.
\]
This proves \eqref{eq:cross-entropy-discrete-continuum-error}. For the adjoint, the finite Euclidean gradient in \eqref{eq:terminal-adjoint-softmax} contains the factor \(1/L\). After multiplying by \(L\) and embedding it as a cellwise density, it represents the continuum logit residual \(\Delta\). The softmax map is smooth and Lipschitz on bounded logit sets, so \(\Delta(p_l)\) is a first-order Riemann sample of \(\Delta(p)\). Applying the assumed first-order consistency \eqref{eq:readout-adjoint-consistency} and boundedness of \(DH(X)^*\) gives \eqref{eq:terminal-adjoint-discrete-continuum-error}.
\end{proof}

Let $\mathscr F_\theta$ be the explicit masked-attention plus tokenwise feed-forward field in \eqref{eq:induced-transformer-field}. The adjoint variable $P_t^{\theta,X_0,Y}\in\mathcal X$ is defined backward in depth time by
\begin{equation}\label{eq:adjoint}
-\partial_t P_t^{\theta,X_0,Y}
=\zeta(t)D\mathscr F_{\theta_t}\!(X_t^\theta)^*P_t^{\theta,X_0,Y},
\qquad
\begin{gathered}
P_T^{\theta,X_0,Y}\text{ given by \eqref{eq:terminal-adjoint-softmax} in the finite-token convention,}\\
\text{or by \eqref{eq:continuum-terminal-adjoint} in the continuum-position convention.}
\end{gathered}
\end{equation}
Here $D\mathscr F_{\theta_t}(X_t^\theta)^*$ is the Hilbert-space adjoint of the Fr\'{e}chet derivative of the complete induced Transformer vector field.  Its nonlocal part includes the derivative of the normalized prefix-attention measures \eqref{eq:attention-induced-measure}, and there is no additional derivative through an independent law-valued coefficient. If the input representation is perturbed as $X_0+\epsilon\xi$, then the first variation is
\begin{equation}\label{eq:first-variation}
\left.\frac{\dd}{\dd\epsilon}\mathscr L(\theta;X_0+\epsilon\xi,Y)\right|_{\epsilon=0}
=\langle P_0^{\theta,X_0,Y},\xi\rangle_{\mathcal X}.
\end{equation}
Thus $\|P_0^{\theta,X_0,Y}(p)\|$ is the first-order loss sensitivity of position $p$ under the controlled continuous-depth Transformer.

\begin{definition}[Positional influence measure and density]\label{def:influence}
For a parameter path $\theta_\cdot$ and a data law $\mathcal D$, define a finite positive measure on Borel sets $B\subset(0,1)$ by
\begin{equation}\label{eq:influence-measure}
\mathcal I_{\theta_\cdot}(B)
:=\E_{(X_0,Y)\sim\mathcal D}
\|\Pi_B P_0^{\theta,X_0,Y}\|_{\mathcal X}^2,
\end{equation}
where $\Pi_B$ is multiplication by $\one_B$. Since $P_0\in L^2$, this measure is absolutely continuous with respect to Lebesgue measure. Its Radon--Nikodym density is denoted by
\begin{equation}\label{eq:unnormalized-influence-density}
I_{\theta_\cdot}(p)
:=\frac{\dd\mathcal I_{\theta_\cdot}}{\dd p}(p)
=\E\|P_0^{\theta,X_0,Y}(p)\|^2
\quad\text{for a.e. }p,
\end{equation}
where the final equality uses any jointly measurable representative. If $\mathcal I_{\theta_\cdot}((0,1))>0$, define
\begin{equation}\label{eq:influence-density}
\mathfrak m_{\theta_\cdot}(p)
=\frac{I_{\theta_\cdot}(p)}{\int_0^1 I_{\theta_\cdot}(q)\dd q}.
\end{equation}
Influence at a single position is therefore understood as the density obtained from shrinking cells: for almost every $p$,
\[
I_{\theta_\cdot}(p)
=\lim_{h\downarrow0}\frac{1}{|B_h(p)|}\mathcal I_{\theta_\cdot}(B_h(p)),
\]
whenever the Lebesgue differentiation theorem applies.
\end{definition}

This definition is deliberately analytic. It does not define influence by an experimental retrieval score, but by the adjoint sensitivity of the continuous-depth predictor. It also differs from the closed-form Ces\`{a}ro influence density of \cite{chowdhury2026birth}: that density is computed directly from an initialized causal decoder abstraction, whereas \eqref{eq:influence-density} is a normalized expected adjoint energy for a controlled continuous-depth predictor evaluated along the deterministic gradient-flow path.

\section{Deterministic Gradient Flow and Influence-Density Dynamics}\label{sec:gd-flow}
This section gives a deterministic continuous-time analysis of primacy, recency, and Lost-in-the-Middle. The only training dynamics in this section are full-batch gradient descent in the small-step limit.

\subsection{Gradient-flow training and the normalized influence density}

Let $\vartheta$ denote the finite-dimensional training state encoding the depth-dependent controls $\theta_t^\vartheta$, the readout $H_\vartheta$, and all other trainable parameters. For a data law $\mathcal D$ on $(X_0,Y)$, define
\begin{equation}\label{eq:gd-risk}
R(\vartheta)
:=
\E_{(X_0,Y)\sim\mathcal D}
\left[
\mathcal\ell\bigl(Y,H_\vartheta(X_T^{\theta^\vartheta})\bigr)
\right],
\end{equation}
where $X_T^{\theta^\vartheta}$ is obtained from the continuous-depth equation \eqref{eq:continuous-depth}. The full-batch gradient-flow limit of gradient descent is
\begin{equation}\label{eq:gd-flow}
\dot\vartheta_s=-\nabla R(\vartheta_s),
\qquad \vartheta_{s=0}=\vartheta_0.
\end{equation}
For a deterministic initialization, the positional influence density along training time is
\begin{equation}\label{eq:gd-influence-density}
m_s^{\rm GD}(p)
=
\frac{I_{\vartheta_s}(p)}
{\displaystyle\int_0^1 I_{\vartheta_s}(q)\dd q},
\end{equation}
where $I_{\vartheta_s}$ is the unnormalized adjoint influence from Definition \ref{def:influence}, computed using the depth-control path $\theta^{\vartheta_s}_\cdot$ and the readout $H_{\vartheta_s}$. Thus the object being tracked is the same positional measure as in \eqref{eq:influence-density}, but evaluated along the single deterministic gradient-flow trajectory.

In practice this follows on any smooth training region where the Transformer maps, the readout, and the layer-normalization surrogate are differentiable in \(\vartheta\), since \(P_0^{\theta^{\vartheta},X_0,Y}\) depends differentiably on \(\vartheta\) through the forward ODE and the backward adjoint equation. Hence \(I_\vartheta(p)=\E\|P_0^{\theta^{\vartheta},X_0,Y}(p)\|^2\) is differentiable for almost every \(p\) by differentiation under the expectation. 

The preceding differentiability statement is the only place where a domination estimate is needed.  We record the estimate explicitly so that the subsequent derivative of the normalized influence density is justified by standard sensitivity bounds for the forward--backward system.

Regularized layer normalization and softmax attention make the explicit finite-token block smooth in states and finite-dimensional parameters on bounded sets. Proposition~\ref{prop:transformer-lipschitz}, however, proves only the state-space regularity needed for well-posedness. The twice Fr\'{e}chet differentiable lifted forward--backward system assumed in the next proposition is an additional sensitivity hypothesis, not a consequence already proved for every admissible control parameterization. Likewise, the $L^4$ condition in Proposition~\ref{prop:influence-l2-differentiability} is separately imposed and is stronger than the $L^2$ well-posedness theory.

\begin{proposition}[$L^1$ differentiability of the influence density]\label{prop:influence-differentiability}
Assume Assumptions \ref{ass:bounded} and \ref{ass:depth-profile}. Let $U$ be a compact parameter neighborhood intersecting the gradient-flow trajectory. Assume that the lifted vector field $G_\vartheta(t,X)=\zeta(t)\mathscr F_{\theta_t^\vartheta}(X)$ is twice continuously Fr\'echet differentiable in $(X,\vartheta)$ on the bounded trajectory set, with its first and second derivatives uniformly bounded for $\vartheta\in U$. Assume a tokenwise linear readout $H_\vartheta(X)=XW_{\rm out}$ with $W_{\rm out}$ in a compact set and
\begin{equation}\label{eq:data-fourth-moment}
\E\bigl[\|X_0\|_{\mathcal X}^4+\|Y\|_{L^2}^4\bigr]<\infty.
\end{equation}
Then the maps $\vartheta\mapsto P_0^\vartheta$ and $\vartheta\mapsto D_\vartheta P_0^\vartheta$ satisfy
\begin{equation}\label{eq:adjoint-parameter-uniform-bound}
\sup_{\vartheta\in U}
\left(
\|P_0^\vartheta\|_{\mathcal X}^2+
\|D_\vartheta P_0^\vartheta\|_{\mathcal L(\R^{d_\vartheta},\mathcal X)}^2
\right)
\le C_{U,T}\bigl(1+\|X_0\|_{\mathcal X}^2+\|Y\|_{L^2}^2\bigr).
\end{equation}
Consequently the map $\vartheta\mapsto I_\vartheta$ is Fr\'echet differentiable from $U$ into $L^1(0,1)$, with
\begin{equation}\label{eq:influence-parameter-gradient}
D_\vartheta I_\vartheta[\eta](p)
=2\E\left\langle
D_\vartheta P_0^\vartheta[\eta](p),P_0^\vartheta(p)
\right\rangle
\quad\text{in }L^1(0,1).
\end{equation}
In particular, the normalization $Z(\vartheta)=\int_0^1I_\vartheta(p)\dd p$ is differentiable whenever finite.
\end{proposition}

\begin{proof}
On the bounded trajectory set, the integral equation for the forward flow and Gronwall's inequality give
\[
\sup_{0\le t\le T}\|X_t^\vartheta\|_{\mathcal X}
\le C_{U,T}(1+\|X_0\|_{\mathcal X}).
\]
Differentiating the flow in $\vartheta$ gives a linear inhomogeneous equation with uniformly bounded coefficients and forcing, hence
\[
\sup_{0\le t\le T}\|D_\vartheta X_t^\vartheta\|
\le C_{U,T}.
\]
For the linear readout, the softmax residual is bounded by $C(1+\|Y\|_{L^2})$, while differentiation with respect to $W_{\rm out}$ introduces at most one factor $X_T^\vartheta$. Therefore
\[
\|P_T^\vartheta\|_{\mathcal X}
+\|D_\vartheta P_T^\vartheta\|
\le C_{U,T}(1+\|X_0\|_{\mathcal X}+\|Y\|_{L^2}).
\]
The backward adjoint and its parameter derivative are linear final-value equations. The second-derivative assumption controls the forcing term
$D_\vartheta[D_XG_\vartheta(t,X_t^\vartheta)]^*P_t^\vartheta$.
Variation of constants and backward Gronwall therefore yield the same linear growth bound at $t=0$; squaring proves \eqref{eq:adjoint-parameter-uniform-bound}.

For $\eta\in\R^{d_\vartheta}$, the pointwise algebraic identity
\[
\|P_0^{\vartheta+h\eta}\|^2-\|P_0^\vartheta\|^2
=2\langle P_0^\vartheta,P_0^{\vartheta+h\eta}-P_0^\vartheta\rangle
+\|P_0^{\vartheta+h\eta}-P_0^\vartheta\|^2
\]
holds for measurable representatives. Integrating in $p$, taking expectations, and applying Cauchy--Schwarz in $L^2(0,1)$ shows that the difference quotient converges in $L^1$ to the right-hand side of \eqref{eq:influence-parameter-gradient}. Indeed,
\[
\|\langle D_\vartheta P_0^\vartheta[\eta],P_0^\vartheta\rangle\|_{L^1}
\le
\|D_\vartheta P_0^\vartheta[\eta]\|_{L^2}
\|P_0^\vartheta\|_{L^2},
\]
and the expectation of this product is finite by \eqref{eq:adjoint-parameter-uniform-bound}, \eqref{eq:data-fourth-moment}, and Cauchy--Schwarz in probability. This proves $L^1$ differentiability and justifies differentiating all regional masses and the normalizing integral.
\end{proof}

\begin{proposition}[$L^2$ differentiability for the squared-density penalty]
\label{prop:influence-l2-differentiability}

In addition to Proposition \ref{prop:influence-differentiability}, assume that\footnote{Here $\mathcal D\otimes\mathrm{Leb}$ is the product of the data law $\mathcal D$ on examples $\omega=(X_0,Y)$ and Lebesgue measure on the position interval $(0,1)$.  Thus, for a jointly measurable random field $Q(\omega,p)$,
$\|Q\|_{L^4(\mathcal D\otimes\mathrm{Leb})}^4
=\E_{\omega\sim\mathcal D}\int_0^1\|Q(\omega,p)\|^4\dd p$.}
\begin{equation}\label{eq:adjoint-l4-sensitivity-assumption}
\vartheta\longmapsto P_0^\vartheta
\quad\text{is continuously Fr\'echet differentiable from $U$ into}\quad
L^4\bigl(\mathcal D\otimes\mathrm{Leb};\R^{d_x}\bigr),
\end{equation}
with
\[
\sup_{\vartheta\in U}
\left(
\|P_0^\vartheta\|_{L^4(\mathcal D\otimes\mathrm{Leb})}
+\|D_\vartheta P_0^\vartheta\|_{\mathcal L(\R^{d_\vartheta},L^4)}
\right)<\infty.
\]
Assume also that
\begin{equation}\label{eq:influence-normalization-away-from-zero}
Z(\vartheta)=\int_0^1I_\vartheta(p)\dd p\ge z_0>0,
\qquad \vartheta\in U.
\end{equation}
Then $\vartheta\mapsto I_\vartheta$ and $\vartheta\mapsto m_\vartheta=I_\vartheta/Z(\vartheta)$ are continuously Fr\'echet differentiable as maps into $L^2(0,1)$.  In particular, for every target density $\nu\in L^2(0,1)$,
\begin{equation}\label{eq:l2-influence-discrepancy}
\Phi_\nu(\vartheta):=\frac12\|m_\vartheta-\nu\|_{L^2(0,1)}^2
\end{equation}
is differentiable and
\begin{equation}\label{eq:l2-influence-discrepancy-derivative}
D\Phi_\nu(\vartheta)[\eta]
=\left\langle m_\vartheta-\nu,
Dm_\vartheta[\eta]\right\rangle_{L^2}.
\end{equation}
\end{proposition}

\begin{remark}[Practical interpretation of the $L^4$ sensitivity hypothesis]\label{rem:l4-sensitivity-practice}
At a fixed finite token length, \eqref{eq:adjoint-l4-sensitivity-assumption} is a natural local regularity condition for a smooth, regularized network: all norms on the finite-dimensional token state are equivalent, and bounded data, bounded parameter sets, regularized layer normalization, and bounded first and second derivatives give local fourth-moment bounds for the adjoint and its parameter sensitivity.  The continuum statement is stronger and is not automatic from the $L^2$ well-posedness theory.  It requires uniform integrability of fourth powers jointly over data and position, so it is reasonable when inputs and readout derivatives have controlled fourth moments and the linearized forward and backward propagators remain uniformly bounded.  It may fail for heavy-tailed data, unbounded readouts or activations, nearly singular normalization, or a sequence-length limit in which fourth moments are not uniform.  In such regimes, the $L^1$ differentiability result of Proposition \ref{prop:influence-differentiability} and regional-mass penalties remain valid under weaker assumptions, whereas the continuum squared-density penalty should be treated as a finite-token objective or used only after empirical fourth-moment and Jacobian-sensitivity diagnostics support the stronger hypothesis.
\end{remark}

\begin{proof}
The map
\[
\mathcal Q:L^4(\mathcal D\otimes\mathrm{Leb})\to L^2(0,1),
\qquad
\mathcal Q(P)(p):=\E\|P(p)\|^2,
\]
is continuously differentiable.  Indeed, Jensen's inequality gives
\[
\|\mathcal Q(P)\|_{L^2}^2
=\int_0^1\bigl(\E\|P(p)\|^2\bigr)^2\dd p
\le \E\int_0^1\|P(p)\|^4\dd p,
\]
and H\"older's inequality gives the bounded derivative
\[
D\mathcal Q(P)[Q](p)=2\E\langle P(p),Q(p)\rangle,
\qquad
\|D\mathcal Q(P)[Q]\|_{L^2}
\le2\|P\|_{L^4}\|Q\|_{L^4}.
\]
The quadratic remainder satisfies
\(\|\mathcal Q(P+Q)-\mathcal Q(P)-D\mathcal Q(P)[Q]\|_{L^2}
\le\|Q\|_{L^4}^2\).
Composing $\mathcal Q$ with \eqref{eq:adjoint-l4-sensitivity-assumption} proves the $L^2$ differentiability of $I_\vartheta$.  Integration is a bounded functional on $L^2(0,1)$, so $Z$ is differentiable; \eqref{eq:influence-normalization-away-from-zero} and the Banach-space quotient rule give the $L^2$ differentiability of $m_\vartheta$.  Finally, the squared $L^2$ norm is continuously differentiable, which proves \eqref{eq:l2-influence-discrepancy-derivative}.
\end{proof}

With Proposition \ref{prop:influence-differentiability} in hand, the unnormalized density is differentiable as an $L^1$ function. Hence the quotient rule below holds in $L^1(0,1)$, and therefore almost everywhere after choosing a representative.
The normalization \(Z_s=\int_0^1 I_{\vartheta_s}(q)\dd q\) is positive whenever the terminal loss has a nonzero adjoint on a set of positive probability; if the model has zero adjoint everywhere, the influence density is degenerate and the positional-bias diagnostic is unnecessary.
Then differentiation of \eqref{eq:gd-influence-density} along \eqref{eq:gd-flow} gives
\begin{equation}\label{eq:gd-mdot}
\partial_s m_s^{\rm GD}(p)
=
-\frac{\nabla_\vartheta I_{\vartheta_s}(p)\cdot\nabla R(\vartheta_s)}{Z_s}
+\frac{m_s^{\rm GD}(p)}{Z_s}
\int_0^1 \nabla_\vartheta I_{\vartheta_s}(q)\cdot\nabla R(\vartheta_s)\dd q .
\end{equation}
The identity \eqref{eq:gd-mdot} is exact but not closed. Indeed, set
\[
r_s(p):=
\begin{cases}
\dfrac{\nabla_\vartheta I_{\vartheta_s}(p)\cdot\nabla R(\vartheta_s)}{I_{\vartheta_s}(p)},& I_{\vartheta_s}(p)>0,\\[0.8em]
0,& I_{\vartheta_s}(p)=0,
\end{cases}
\qquad
\overline r_s:=\int_0^1 m_s^{\rm GD}(q)r_s(q)\dd q .
\]
On the positive set, $\partial_s I=-r_sI$. On $\{I=0\}$ the original quotient-rule identity \eqref{eq:gd-mdot} is used; since $I_\vartheta(p)\ge0$ and is differentiable in the parameter, its parameter gradient vanishes at an interior zero. Thus the convention above makes $m_sr_s=0$ on the zero set and the replicator form
\[
\partial_s m_s^{\rm GD}(p)=m_s^{\rm GD}(p)\bigl(\overline r_s-r_s(p)\bigr)
\]
holds almost everywhere.
Thus the normalized mass at position \(p\) increases exactly when its relative decay rate \(r_s(p)\) is smaller than the density-weighted average \(\overline r_s\); equivalently, other positions are being suppressed faster than position \(p\). This is the training-time form of positional bias: the bias is determined not only by the current attention weights, but also by the way the full gradient-flow vector field redistributes normalized adjoint energy across positions.

\subsection{Quantifying primacy, recency, and the middle deficit}

For $\delta\in(0,1/2)$, use the half-open convention
\[
E_L=[0,\delta),\qquad E_M=[\delta,1-\delta),\qquad E_R=[1-\delta,1].
\]
The continuum boundary points have zero measure, while the convention gives an unambiguous finite-token assignment. Define
\begin{align}
\overline m_L^{\rm GD}(s,\delta)
&:=\frac1\delta\int_{E_L} m_s^{\rm GD}(p)\dd p, \label{eq:gd-left-average}\\
\overline m_M^{\rm GD}(s,\delta)
&:=\frac1{1-2\delta}\int_{E_M}m_s^{\rm GD}(p)\dd p,\label{eq:gd-middle-average}\\
\overline m_R^{\rm GD}(s,\delta)
&:=\frac1\delta\int_{E_R}m_s^{\rm GD}(p)\dd p.\label{eq:gd-right-average}
\end{align}
since $m_s^{\rm GD}$ integrates to one over an interval of length one, these are \emph{regional densities}, not regional probability masses; an average may therefore exceed one when influence is concentrated in a short region.

Primacy means $\overline m_L^{\rm GD}>\overline m_M^{\rm GD}$, recency means $\overline m_R^{\rm GD}>\overline m_M^{\rm GD}$, and Lost-in-the-Middle means both. Define the unregularized boundary-minus-middle gap
\begin{equation}\label{eq:gd-lim-gap}
\Delta_\delta(s):=
\min\{\overline m_L^{\rm GD},\overline m_R^{\rm GD}\}-\overline m_M^{\rm GD},
\end{equation}
the stabilized index
\begin{equation}\label{eq:gd-lim-index}
\mathsf{LIM}_{\delta}(s)
:=\frac{\Delta_\delta(s)}
{\min\{\overline m_L^{\rm GD},\overline m_R^{\rm GD}\}+\varepsilon_0},
\qquad \varepsilon_0>0,
\end{equation}
and the symmetric scale-free contrast
\begin{equation}\label{eq:gd-lim-symmetric-contrast}
\mathsf{SC}_{\delta}(s)
:=\frac{\Delta_\delta(s)}
{\min\{\overline m_L^{\rm GD},\overline m_R^{\rm GD}\}+\overline m_M^{\rm GD}+\varepsilon_0}.
\end{equation}
All three quantities have the same sign. The magnitude of $\mathsf{LIM}_\delta$ depends on $\varepsilon_0$, so numerical reports in Section \ref{sec:simulations} state its value and also report $\Delta_\delta$ and $\mathsf{SC}_\delta$.

\subsection{Primacy from the Volterra adjoint cone}
The causal support condition in \eqref{eq:continuum-causal-kernel} implies
\begin{equation}\label{eq:causal-support-recalled}
A_\theta^h(p,q;X)=0
\qquad\text{whenever }q>p.
\end{equation}
Equations \eqref{eq:continuum-causal-attention} and \eqref{eq:continuum-causal-kernel} show that every quantity retained at output position $p$ depends only on the prefix $X|_{[0,p]}$.  The feed-forward sublayer is tokenwise, and any retained positional statistic is likewise prefix-causal.  Differentiation therefore preserves this lower-triangular dependence, so no additional causal-derivative-closure assumption is needed.
Consequently, the Fr\'echet derivative of causal attention is lower triangular in the position variable.  More precisely, a linearized causal component acting on a perturbation \(\xi\) has the Volterra--local form
\begin{equation}\label{eq:volterra-linearization}
(\mathcal K_t\xi)(p)
=\int_0^p K_t(p,q)\xi(q)\dd q
+\mathcal B_t(p)\xi(p),
\end{equation}
where the integral term transports key and value perturbations from positions \(q\le p\), while \(\mathcal B_t(p)\) collects terms depending only on the perturbation at the output position \(p\).  We now derive this decomposition directly from the normalized attention map.

For one head, set
\begin{equation}\label{eq:one-head-causal-map}
\mathcal T_h(X)(p)
:=\int_0^p A_h(p,q;X)\mathsf V_h(X(q))\dd q,
\qquad p>0,
\end{equation}
where
\[
\mathsf Q_h(z):=W_{Q,h}\mathrm{LN}(z),\qquad
\mathsf K_h(z):=W_{K,h}\mathrm{LN}(z),\qquad
\mathsf V_h(z):=W_{V,h}\mathrm{LN}(z).
\]
Writing
\begin{equation}\label{eq:one-head-score-and-kernel}
s_h(p,q;X)
:=\frac{\langle\mathsf Q_h(X(p)),\mathsf K_h(X(q))\rangle}{\sqrt{d_k}},
\qquad
A_h(p,q;X)
:=\frac{e^{s_h(p,q;X)}}{\int_0^p e^{s_h(p,r;X)}\dd r},
\end{equation}
the first variation of \(\mathcal T_h\) is
\begin{align}
D\mathcal T_h(X)[\xi](p)
&=\int_0^p A_h(p,q;X)
D\mathsf V_h(X(q))[\xi(q)]\dd q \notag\\
&\quad+\int_0^p DA_h(p,q;X)[\xi]\,
\mathsf V_h(X(q))\dd q .
\label{eq:attention-first-variation-schematic}
\end{align}
The score variation separates into a query contribution, depending on \(\xi(p)\), and a key contribution, depending on \(\xi(q)\):
\begin{align}
\delta s_{h,Q}(p,q)[\eta]
&:=\frac{\langle D\mathsf Q_h(X(p))[\eta],\mathsf K_h(X(q))\rangle}{\sqrt{d_k}},\notag\\
\delta s_{h,K}(p,q)[\eta]
&:=\frac{\langle\mathsf Q_h(X(p)),D\mathsf K_h(X(q))[\eta]\rangle}{\sqrt{d_k}}.
\label{eq:head-score-variation}
\end{align}
Thus
\[
Ds_h(p,q;X)[\xi]
=\delta s_{h,Q}(p,q)[\xi(p)]
+\delta s_{h,K}(p,q)[\xi(q)].
\]
Differentiating both the numerator and the normalizing denominator in \eqref{eq:one-head-score-and-kernel} gives
\begin{equation}\label{eq:normalized-kernel-variation}
DA_h(p,q;X)[\xi]
=A_h(p,q;X)
\left(
Ds_h(p,q;X)[\xi]
-\int_0^p A_h(p,r;X)Ds_h(p,r;X)[\xi]\dd r
\right).
\end{equation}
Let
\begin{equation}\label{eq:one-head-attention-mean-value}
\overline V_h(p):=\mathcal T_h(X)(p)
=\int_0^p A_h(p,r;X)\mathsf V_h(X(r))\dd r.
\end{equation}
Substituting \eqref{eq:normalized-kernel-variation} into \eqref{eq:attention-first-variation-schematic} and using the normalization
\(\int_0^p A_h(p,q;X)\dd q=1\) yields the exact identity
\begin{equation}\label{eq:attention-linearization-exact}
D\mathcal T_h(X)[\xi](p)
=\mathsf B_{h,Q}(p)\xi(p)
+\int_0^p\mathsf K_{h,X}(p,q)\xi(q)\dd q,
\end{equation}
where the nonlocal causal kernel is the linear map
\begin{align}
\mathsf K_{h,X}(p,q)\eta
:=A_h(p,q;X)\Bigl(&D\mathsf V_h(X(q))[\eta]\notag\\
&+\delta s_{h,K}(p,q)[\eta]
\bigl(\mathsf V_h(X(q))-\overline V_h(p)\bigr)\Bigr),
\label{eq:attention-nonlocal-kernel}
\end{align}
and the local query operator is
\begin{equation}\label{eq:attention-local-operator}
\mathsf B_{h,Q}(p)\eta
:=\int_0^p A_h(p,q;X)\,
\delta s_{h,Q}(p,q)[\eta]
\bigl(\mathsf V_h(X(q))-\overline V_h(p)\bigr)\dd q.
\end{equation}
All layer-normalization derivatives are contained in
\(D\mathsf Q_h\), \(D\mathsf K_h\), and \(D\mathsf V_h\).  Equation \eqref{eq:attention-nonlocal-kernel} is supported only on \(q\le p\), whereas \eqref{eq:attention-local-operator} multiplies \(\xi(p)\) and is therefore local in position.  The formulas are stated for \(p>0\); at the left endpoint the derivative is understood through the operator trace in \eqref{eq:continuum-attention-left-trace}, consistently with the treatment of the nonlinear attention map.

After concatenating the headwise derivatives, applying the output projection \(W_O\), adding the tokenwise feed-forward derivative, and including the prefix-causal auxiliary contributions retained in the modeled architecture, one obtains
\begin{equation}\label{eq:multihead-volterra-derivative}
D\mathscr F_{\theta_t}(X_t)[\xi](p)
=\int_0^p K_t(p,q)\xi(q)\dd q
+\mathcal B_t(p)\xi(p).
\end{equation}
Here \(K_t\) contains the projected headwise kernels \eqref{eq:attention-nonlocal-kernel} and any other causal nonlocal derivative terms, while \(\mathcal B_t\) contains the projected query operators \eqref{eq:attention-local-operator}, the feed-forward derivative, and other tokenwise local terms.  The residual identity is not part of \(D\mathscr F_{\theta_t}(X_t)\) and is therefore not included in \(\mathcal B_t\).  It enters the derivative of a discrete residual step as the separate identity factor and appears in continuous depth as the direct terminal term \(P_T\) in the Duhamel formula \eqref{eq:continuous-adjoint-duhamel}.  This separation is necessary for distinguishing the Volterra primacy channel from residual transmission.

The Hilbert-space adjoint of \eqref{eq:volterra-linearization} satisfies
\begin{equation}\label{eq:volterra-adjoint}
(\mathcal K_t^*a)(q)
=
\int_q^1 K_t(p,q)^*a(p)\dd p+\mathcal B_t(q)^*a(q).
\end{equation}
Thus an early perturbation $q$ has a larger future causal cone $\{p:p\ge q\}$ than a middle perturbation, and the adjoint formula shows this transport explicitly: a terminal covector value \(a(p)\) at a later position \(p\ge q\) contributes to the earlier position \(q\) through the term \(K_t(p,q)^*a(p)\), and the total contribution is obtained by integrating over all future positions \(p\in[q,1]\). Thus loss sensitivity located later in the context is pushed backward through the causal edges to every earlier position that could have influenced it. This is the deterministic adjoint mechanism behind primacy.

One can make the cone effect visible through the cone mass
\begin{equation}\label{eq:cone-mass}
\mathcal C(q)
:=
\int_q^1 \|K_t(p,q)\|_{\rm op}^2\dd p .
\end{equation}
In practice one can estimate \(\mathcal C(q)\) from the finite-token Jacobian of the attention block. Let \(\widehat K_t(i,j)\) denote the off-diagonal block of the Jacobian mapping a perturbation at token \(j\) into the attention output at token \(i\), with \(j\le i\). Then the discrete cone-mass estimator is
\begin{equation}\label{eq:finite-cone-mass-estimator}
\widehat{\mathcal C}_t(j):=\sum_{i=j}^{L}\Delta p\,\|\widehat K_t(i,j)\|_{\rm op}^2,\qquad \Delta p:=L^{-1}.
\end{equation}
When forming the full Jacobian is too expensive, the operator norms can be replaced by Hutchinson-type probe estimates, for independent unit-variance probes \(\xi^{(r)}\) supported near token \(j\):
\begin{equation}\label{eq:probe-cone-mass-estimator}
\widehat{\mathcal C}_t(j)\approx \frac1{N_{\rm probe}}\sum_{r=1}^{N_{\rm probe}}\sum_{i=j}^{L}\Delta p\,\|(D\mathrm{Attn}_t(X)\xi^{(r)})(i)\|^2.
\end{equation}
This gives a measurable finite-token proxy for \eqref{eq:cone-mass}.

The cone length alone gives only an upper bound by Cauchy--Schwarz and does not rule out vector cancellation.  Moreover, both the kernel and adjoint are data-dependent.  Write $\omega=(X_0,Y)$ and
\[
K_{t,\omega}(p,q):=K_t(p,q;X_{t,\omega}),
\qquad P_{t,\omega}:=P_t^{\theta,X_0,Y}.
\]
For a fixed depth time the random-kernel identity is
\begin{align}
Q_t(q)
&:=\E\left\|\int_q^1K_{t,\omega}(p,q)^*P_{t,\omega}(p)\dd p\right\|^2\notag\\
&=\int_q^1\!\int_q^1
\operatorname{Re}\E\operatorname{tr}\!\left[
K_{t,\omega}(p,q)^*P_{t,\omega}(p)
P_{t,\omega}(r)^*K_{t,\omega}(r,q)
\right]\dd p\dd r.
\label{eq:exact-cone-covariance-energy}
\end{align}
The kernel cannot be pulled outside the expectation unless it is deterministic or an appropriate conditional covariance is used.
In the present setting the depth-control path $\theta_t$ is deterministic once the full-batch training trajectory is fixed, but the kernel is evaluated at the example-dependent hidden state $X_{t,\omega}$.  Hence $K_{t,\omega}(p,q)=K_t(p,q;X_{t,\omega})$ is generally random under $\omega\sim\mathcal D$ (through $X_0$; the label $Y$ enters the adjoint rather than the forward state).  The kernel is deterministic only in special cases, such as a state-independent linearized attention operator, a data distribution concentrated on a single forward trajectory, or an analysis conditioned on a fixed input trajectory.  Therefore \eqref{eq:exact-cone-covariance-energy}, with the kernel retained inside the expectation, is the appropriate identity for the modeled data-averaged influence.

The actual cone channel in the input adjoint is integrated over depth:
\begin{equation}\label{eq:depth-integrated-cone-channel}
P_{\rm cone,\omega}(q)
:=\int_0^T\zeta(t)\int_q^1
K_{t,\omega}(p,q)^*P_{t,\omega}(p)\dd p\dd t.
\end{equation}
Define the joint cross-depth integrand
\begin{align}
\Gamma(t,u,p,r;q)
:={}&\zeta(t)\zeta(u)\operatorname{Re}\E\operatorname{tr}\!\left[
K_{t,\omega}(p,q)^*P_{t,\omega}(p)\right.\notag\\[-0.3em]
&\left.\hspace{8em}{}\times
P_{u,\omega}(r)^*K_{u,\omega}(r,q)
\right].
\label{eq:joint-cross-depth-cone-integrand}
\end{align}
Whenever this integrand is absolutely integrable, Fubini's theorem gives the exact identity
\begin{align}
Q_{\rm cone}(q)
&:=\E\|P_{\rm cone,\omega}(q)\|^2\notag\\
&=\int_0^T\!\int_0^T\!\int_q^1\!\int_q^1
\Gamma(t,u,p,r;q)\dd p\dd r\dd t\dd u.
\label{eq:full-depth-cone-energy}
\end{align}

\begin{proposition}[Conditional primacy under joint non-cancellation]
\label{prop:primacy-noncancellation}
Let $E_L=[0,\delta)$ and $E_M=[\delta,1-\delta)$, and define
\begin{equation}\label{eq:cone-geometric-region-factor}
A_B:=\frac1{|B|}\int_B(1-q)^2\dd q.
\end{equation}
Assume absolute integrability in \eqref{eq:full-depth-cone-energy} and constants $\kappa_L,K_M>0$ such that
\begin{align}
\Gamma(t,u,p,r;q)&\ge\kappa_L
&&\text{for a.e. }q\in E_L, (t,u)\in(0,T)^2, (p,r)\in(q,1)^2,
\label{eq:cone-left-joint-coercivity}\\
|\Gamma(t,u,p,r;q)|&\le K_M
&&\text{for a.e. }q\in E_M, (t,u)\in(0,T)^2, (p,r)\in(q,1)^2.
\label{eq:cone-middle-joint-upper-bound}
\end{align}
Then
\begin{align}
\frac1{|E_L|}\int_{E_L}Q_{\rm cone}(q)\dd q
&\ge \kappa_LT^2A_{E_L},
\label{eq:cone-left-actual-lower-bound}\\
\frac1{|E_M|}\int_{E_M}Q_{\rm cone}(q)\dd q
&\le K_MT^2A_{E_M}.
\label{eq:cone-middle-actual-upper-bound}
\end{align}
Consequently, if
\begin{equation}\label{eq:cone-actual-primacy-condition}
\kappa_LA_{E_L}>K_MA_{E_M},
\end{equation}
the \emph{actual depth-integrated cone energy}, not merely its lower bound, has a strict left--middle regional advantage.
\end{proposition}

\begin{proof}
For $q\in E_L$, integrate \eqref{eq:cone-left-joint-coercivity} over $(0,T)^2\times(q,1)^2$ in \eqref{eq:full-depth-cone-energy}; this gives
$Q_{\rm cone}(q)\ge\kappa_LT^2(1-q)^2$.
For $q\in E_M$, absolute integrability and \eqref{eq:cone-middle-joint-upper-bound} give
$Q_{\rm cone}(q)\le K_MT^2(1-q)^2$.
Regional averaging proves \eqref{eq:cone-left-actual-lower-bound}--\eqref{eq:cone-middle-actual-upper-bound}, and \eqref{eq:cone-actual-primacy-condition} gives the strict comparison.
\end{proof}

Without joint cross-data and cross-depth non-cancellation, and without an upper control on the middle channel, the Volterra formula identifies a possible primacy channel but does not by itself prove primacy.

\subsection{Recency from the residual identity channel}

Residual connections provide a distinct transmission channel: they preserve terminal sensitivity through depth, but they generate a right-boundary advantage only when the terminal readout or task protocol is already right biased.

For the split residual layer, write its exact state Jacobian on a bounded trajectory as
\begin{equation}\label{eq:residual-jacobian-remainder}
D_{X_k}X_{k+1}=I+\varepsilon A_k+E_k,
\qquad
A_k:=\zeta(t_k)D\mathscr F_{\theta_k}(X_k),
\qquad
\|E_k\|_{\rm op}\le C\varepsilon^2.
\end{equation}
The remainder bound is the differentiated form of the split-layer defect. Duality gives
\begin{equation}\label{eq:residual-adjoint-step}
P_k=(I+B_k)P_{k+1},
\qquad B_k:=\varepsilon A_k^*+E_k^*.
\end{equation}
Hence the exact backward propagator is
\begin{equation}\label{eq:residual-adjoint-product}
P_0=(I+B_0)(I+B_1)\cdots(I+B_{M-1})P_M,
\end{equation}
and its finite algebraic expansion is
\begin{equation}\label{eq:identity-plus-interaction-expansion}
P_0=P_M+
\sum_{r=1}^{M}\ \sum_{0\le k_1<\cdots<k_r\le M-1}
B_{k_1}\cdots B_{k_r}P_M.
\end{equation}
The first term $P_M$ is present even if the attention and feed-forward derivative terms are set to zero. It is therefore the purely residual identity channel. In the continuous-depth notation, the same decomposition follows from Duhamel's formula:
\begin{equation}\label{eq:continuous-adjoint-duhamel}
P_0(p)
=P_T(p)+\int_0^T \zeta(t)\bigl(D\mathscr F_{\theta_t}(X_t)^*P_t\bigr)(p)\dd t .
\end{equation}
Thus terminal sensitivity at position $p$ is copied to depth zero at zeroth order, while the Volterra and feed-forward interactions only add the integral correction.

Let $E_R=[1-\delta,1]$, $E_M=[\delta,1-\delta)$, and let $\Pi_R$ and $\Pi_M$ denote multiplication by the indicators of these two sets. Recency is not imposed as an additional assumption; in the present framework it is the part of the measured influence density that is explained by the residual transmission of terminal adjoint energy. The residual contribution is therefore defined by
\begin{equation}\label{eq:residual-influence-component}
I_{\rm res}(s,p)
:=\E_{(X_0,Y)\sim\mathcal D}
\left[\left\|P_T^{\theta^{\vartheta_s},X_0,Y}(p)\right\|^2\right],
\end{equation}
and its right-middle contrast is the observable quantity
\begin{equation}\label{eq:residual-recency-contrast}
\mathsf R_{\rm res}(s,\delta)
:=
\frac{1}{\delta}\int_{1-\delta}^1 I_{\rm res}(s,p)\dd p
-
\frac{1}{1-2\delta}\int_{\delta}^{1-\delta} I_{\rm res}(s,p)\dd p .
\end{equation}
A positive value of \eqref{eq:residual-recency-contrast} means exactly that the average terminal adjoint energy density on \(E_R\) exceeds the corresponding average on \(E_M\). Since the residual identity part in \eqref{eq:continuous-adjoint-duhamel} transmits \(P_T\) directly into \(P_0\), any right-heavy terminal energy produced by the loss or readout is already a right-heavy input-sensitivity component before the Volterra correction is added.

In an autoregressive decoder this right-boundary contribution can be expressed directly from the next-token loss without assuming a priori that the query locations are terminal. 
Let
\[
\alpha^L=(\alpha_1^L,\ldots,\alpha_L^L),
\qquad \alpha_i^L\ge0,
\qquad \sum_{i=1}^{L}\alpha_i^L=1,
\]
be the empirical weight with which the training or evaluation objective reads out the next-token loss at query location \(i\). These weights are useful in practice since different protocols use different query sets: full teacher-forced training averages over many positions, prompt evaluation may use only the last query, and benchmark losses may average over a selected subset. Mathematically, full teacher-forced training corresponds to \(\alpha_i^L=1/L\) for all available query positions, because the loss averages the next-token prediction error at every token. Last-token prompt evaluation corresponds to \(\alpha_L^L=1\) and \(\alpha_i^L=0\) for \(i<L\), because the benchmark conditions on the whole prompt and scores only the prediction after the final prompt token. More generally, if a benchmark scores a subset \(S_L\subset\{1,\ldots,L\}\), then \(\alpha_i^L=|S_L|^{-1}\one_{\{i\in S_L\}}\). The finite-context autoregressive loss can therefore be written as the weighted objective
\begin{equation}\label{eq:autoregressive-terminal-loss}
\mathcal\ell_{\rm AR}(Y,H(X))
=-\sum_{i=1}^{L}\alpha_i^L
Y_{i+1}^{\top}\log\operatorname{softmax}(H(X)_i),
\end{equation}
Equation~\eqref{eq:autoregressive-terminal-loss} uses an observed context of $L$ tokens together with an additional continuation target $Y_{L+1}$, so it has $L$ scored query positions. Ordinary teacher forcing on a sequence containing exactly $L$ observed tokens instead scores only $i=1,\ldots,L-1$ with targets $Y_2,\ldots,Y_L$; in that convention set $\alpha_L^L=0$ and renormalize the weights over the available queries. The continuation convention is used below only when the task supplies or evaluates the token following the full context.
The logit residual entering the terminal adjoint is therefore
\begin{equation}\label{eq:autoregressive-logit-residual}
G_i^{\rm AR}
=
\alpha_i^L
\bigl(\operatorname{softmax}(H(X_T)_i)-Y_{i+1}\bigr),
\qquad i=1,\ldots,L,
\end{equation}
and hence
\begin{equation}\label{eq:autoregressive-terminal-adjoint}
P_T
=DH(X_T)^*G^{\rm AR}.
\end{equation}
Thus recency is tied to the measured distribution of the query weights $\alpha^L$, not to an extra structural assumption on their support. Define the discrete right and middle query masses
\begin{equation}\label{eq:discrete-right-middle-sets}
E_R^L(\delta):=\{i: i/L\in[1-\delta,1]\},
\qquad
E_M^L(\delta):=\{i: i/L\in[\delta,1-\delta)\}.
\end{equation}
\begin{equation}\label{eq:query-mass-discrete}
\omega_R^L(\delta):=\sum_{i\in E_R^L(\delta)}\alpha_i^L,
\qquad
\omega_M^L(\delta):=\sum_{i\in E_M^L(\delta)}\alpha_i^L.
\end{equation}
A right-biased query/readout geometry means
\begin{equation}\label{eq:right-biased-query-mass}
\omega_R^L(\delta)>\omega_M^L(\delta),
\end{equation}
or, more sharply for squared adjoint energy, that
\begin{equation}\label{eq:right-biased-squared-query-mass}
\frac1{|E_R^L(\delta)|}\sum_{i\in E_R^L(\delta)}(\alpha_i^L)^2
>
\frac1{|E_M^L(\delta)|}\sum_{i\in E_M^L(\delta)}(\alpha_i^L)^2.
\end{equation}
This condition is an observable property of the loss/readout placement. For example, in last-token prompt evaluation, namely the application scenario in which a model is given the full prompt or retrieved context and the metric scores only the next-token or answer prediction made from the final query position, one has $\alpha_{L}^L=1$ and $\alpha_i^L=0$ for $i<L$, so \eqref{eq:right-biased-query-mass} holds for every fixed $\delta>0$ once $L$ is large. 

In fully teacher-forced training with uniform token averaging, each query contributes equally, so \(\alpha_i^L=1/L\). Then \(\omega_R^L(\delta)=|E_R^L(\delta)|/L\approx\delta\) and \(\omega_M^L(\delta)=|E_M^L(\delta)|/L\approx1-2\delta\). Moreover, the per-position squared averages satisfy \( |E_R^L|^{-1}\sum_{i\in E_R^L}(\alpha_i^L)^2=|E_M^L|^{-1}\sum_{i\in E_M^L}(\alpha_i^L)^2=L^{-2}\), up to endpoint discretization. Hence the terminal-loss channel by itself does not select the right boundary. This is why the present formulation separates the residual identity mechanism from the task protocol: recency appears only when the measured query/readout weights or residual factors actually make \eqref{eq:residual-recency-contrast} positive.

In the simple tokenwise readout case $H(X)_i=C_iX(i)$, one obtains
\begin{equation}\label{eq:tokenwise-terminal-adjoint}
P_T(i)
=C_i^*G_i^{\rm AR}
=\alpha_i^L C_i^*
\bigl(\operatorname{softmax}(H(X_T)_i)-Y_{i+1}\bigr).
\end{equation}
Consequently the right-minus-middle residual energy satisfies
\begin{align}
&\frac1{|E_R^L|}\sum_{i\in E_R^L}\E\|P_T(i)\|^2
-
\frac1{|E_M^L|}\sum_{i\in E_M^L}\E\|P_T(i)\|^2 \notag\\
&\quad =
\frac1{|E_R^L|}\sum_{i\in E_R^L}(\alpha_i^L)^2
\E\left\|C_i^*\bigl(\operatorname{softmax}(H(X_T)_i)-Y_{i+1}\bigr)\right\|^2 \notag\\
&\qquad -
\frac1{|E_M^L|}\sum_{i\in E_M^L}(\alpha_i^L)^2
\E\left\|C_i^*\bigl(\operatorname{softmax}(H(X_T)_i)-Y_{i+1}\bigr)\right\|^2 .\label{eq:finite-recency-energy-contrast}
\end{align}
To make this comparison explicit, define
\[
\rho_i:=\E\left\|C_i^*\bigl(\operatorname{softmax}(H(X_T)_i)-Y_{i+1}\bigr)\right\|^2.
\]
Then \eqref{eq:finite-recency-energy-contrast} is
\[
\frac1{|E_R^L|}\sum_{i\in E_R^L}(\alpha_i^L)^2\rho_i
-
\frac1{|E_M^L|}\sum_{i\in E_M^L}(\alpha_i^L)^2\rho_i.
\]
If the readout residual energies are comparable, say \(0<\rho_-\le \rho_i\le \rho_+<\infty\) and \(\rho_+/\rho_-\) is close to one across the two regions, then the sign is governed mainly by the difference of the squared query-weight averages in \eqref{eq:right-biased-squared-query-mass}. In this notation, the informal statement that the last observed tokens are closest to the terminal query/readout locations should be read as a coefficient-level condition: after the query weights and readout residual energies are combined, the effective coefficients \((\alpha_i^L)^2\rho_i\) have a larger average over \(E_R^L(\delta)\) than over \(E_M^L(\delta)\).

For a local readout
\begin{equation}\label{eq:local-readout-finite}
H(X)_i=\sum_{j=1}^{L}C_{ij}X(j),
\qquad C_{ij}=0\quad\text{if } |i-j|>r_H,
\end{equation}
Here the condition \(C_{ij}=0\) for \(|i-j|>r_H\) models a readout whose logit at query \(i\) depends only on hidden states within a window of radius \(r_H\) around \(i\). This includes tokenwise readout as the case \(r_H=0\), and it also covers local smoothing, local pooling, or finite-window output heads in practice: these are readouts in which the logit at position \(i\) is formed from a short neighborhood of hidden states rather than from exactly one hidden state. It is a locality assumption on the readout operator, not an additional assumption on the causal attention mask.
The terminal adjoint is
\begin{equation}\label{eq:local-readout-adjoint}
P_T(j)=\sum_{i=1}^{L}C_{ij}^*G_i^{\rm AR},
\end{equation}
and hence
\begin{equation}\label{eq:local-readout-support}
\operatorname{supp} P_T
\subset
\{j:\operatorname{dist}(j,\operatorname{supp}\alpha^L)\le r_H\}.
\end{equation}
Thus, if \(\alpha^L\) is concentrated in \(E_R^L(\delta)\), then \(G_i^{\rm AR}\) is concentrated at right-side query indices. The formula \(P_T(j)=\sum_i C_{ij}^*G_i^{\rm AR}\) can move this support only to indices \(j\) with \(|i-j|\le r_H\). Therefore the terminal adjoint remains right-biased, with at most an \(r_H\)-neighborhood of spatial spreading.

In continuum notation, let $\chi\ge0$ be a continuum query/readout density, namely the continuum limit of the discrete weights \(\alpha_i^L\) so that \(\chi(p)\dd p\) gives the fraction of loss/readout mass assigned to query positions in a small interval around \(p\) with $\int_0^1\chi(p)\dd p=1$. The autoregressive readout loss is
\begin{equation}\label{eq:continuum-ar-loss}
\mathcal\ell_{\rm AR}(Y,H(X))
=-\int_0^1 \chi(p)Y(p)^\top
\log\operatorname{softmax}(H(X)(p))\dd p,
\end{equation}
and the continuum terminal adjoint is
\begin{equation}\label{eq:continuum-ar-terminal-adjoint}
P_T
=DH(X_T)^*\Bigl(\chi(\cdot)
\bigl(\operatorname{softmax}(H(X_T)(\cdot))-Y(\cdot)\bigr)\Bigr).
\end{equation}
The right-bias condition is now the measurable inequality
\begin{equation}\label{eq:continuum-query-right-bias}
\frac1\delta\int_{1-\delta}^1\chi(p)^2\dd p
>
\frac1{1-2\delta}\int_{\delta}^{1-\delta}\chi(p)^2\dd p,
\end{equation}
possibly after multiplying by the local readout-residual energy. Under \eqref{eq:continuum-query-right-bias}, the residual component \eqref{eq:residual-influence-component} is naturally larger near the right boundary, up to the spatial spread of $DH(X_T)^*$. The theory only requires the measurable contrast \eqref{eq:residual-recency-contrast}; the autoregressive formulas \eqref{eq:autoregressive-terminal-loss}--\eqref{eq:continuum-query-right-bias} explain one standard mechanism by which that contrast can become positive.

For a possibly unbounded data distribution, define the pathwise operator size
\begin{equation}\label{eq:adjoint-operator-bound}
\Lambda_s(\omega):=
\sup_{0\le t\le T}|\zeta(t)|\,
\|D\mathscr F_{\theta_t^{\vartheta_s}}(X_{t,\omega})^*\|_{\mathcal L(\mathcal X)},
\qquad
c_s(\omega):=e^{T\Lambda_s(\omega)}-1.
\end{equation}
Whenever $\Lambda_s(\omega)<\infty$, backward Gronwall gives the pathwise estimate
\begin{equation}\label{eq:residual-stability-bound}
\|P_0(\omega)-P_T(\omega)\|_{\mathcal X}
\le c_s(\omega)\|P_T(\omega)\|_{\mathcal X}.
\end{equation}

\begin{proposition}[Regional squared-energy persistence of recency]
\label{prop:regional-recency-persistence}
Let
\[
\mathcal A_B(P):=\frac1{|B|}\E\|\Pi_BP\|_{\mathcal X}^2
\]
and assume
\begin{equation}\label{eq:recency-moment-condition}
\E\bigl[(2c_s+c_s^2)\|P_T\|_{\mathcal X}^2\bigr]<\infty.
\end{equation}
Then, for every measurable region $B$ of positive length,
\begin{equation}\label{eq:regional-squared-energy-stability}
|\mathcal A_B(P_0)-\mathcal A_B(P_T)|
\le \frac1{|B|}
\E\bigl[(2c_s+c_s^2)\|P_T\|_{\mathcal X}^2\bigr].
\end{equation}
Consequently the terminal right bias survives whenever
\begin{equation}\label{eq:regional-recency-survival-condition}
\mathsf R_{\rm res}(s,\delta)
>
\left(\frac1\delta+\frac1{1-2\delta}\right)
\E\bigl[(2c_s+c_s^2)\|P_T\|_{\mathcal X}^2\bigr].
\end{equation}
Under this condition, $\mathcal A_{E_R}(P_0)>\mathcal A_{E_M}(P_0)$.
\end{proposition}

\begin{proof}
Set $d=P_0-P_T$. The identity
\[
\|\Pi_B(P_T+d)\|^2-\|\Pi_BP_T\|^2
=2\operatorname{Re}\langle\Pi_BP_T,\Pi_Bd\rangle+\|\Pi_Bd\|^2
\]
and \eqref{eq:residual-stability-bound} give pathwise
\[
\left|\|\Pi_BP_0\|^2-\|\Pi_BP_T\|^2\right|
\le(2c_s+c_s^2)\|P_T\|^2.
\]
Taking expectations proves \eqref{eq:regional-squared-energy-stability}; applying it to $E_R$ and $E_M$ proves the survival statement.
\end{proof}
Thus the residual mechanism preserves a right-heavy terminal component only when its squared-energy gap dominates a quantitatively matched interaction error.

The primacy and recency channels are therefore mathematically distinct. Primacy is generated by the depth-integrated nonlocal Volterra adjoint term
\begin{equation}\label{eq:primacy-channel-term}
P_{\rm cone}(q)
= \int_0^T\zeta(t)\int_q^1 K_t(p,q)^*P_t(p)\dd p\dd t,
\end{equation}
whose integration domain is largest for early $q$. Recency is transmitted by the residual identity term
\begin{equation}\label{eq:recency-channel-term}
P_{\rm res}(p)\sim P_T(p),
\end{equation}
whose size is controlled by the terminal readout and loss geometry. Hence one may define the boundary channel strengths
\begin{align}
\mathsf P_L(s,\delta)&:=\int_0^\delta \E\|P_{\rm cone}(p)\|^2\dd p,\label{eq:primacy-strength}\\
\mathsf P_R(s,\delta)&:=\int_{1-\delta}^1 \E\|P_{\rm res}(p)\|^2\dd p.\label{eq:recency-strength}
\end{align}
A model can have strong primacy with weak recency when $\mathsf P_L(s,\delta)$ dominates $\mathsf P_R(s,\delta)$, strong recency with weak primacy in the opposite regime, or both when the Volterra cone and the residual terminal channel are simultaneously large. Lost-in-the-Middle is most pronounced when both boundary strengths dominate the corresponding middle influence components.

\subsection{Lost-in-the-Middle as an exact channel-energy comparison}

Use Duhamel's formula and the local/nonlocal splitting of the adjoint generator to write
\begin{equation}\label{eq:adjoint-channel-vector-decomposition}
P_0=P_{\rm res}+P_{\rm cone}+P_{\rm loc},
\end{equation}
where $P_{\rm res}:=P_T$, while $P_{\rm cone}$ and $P_{\rm loc}$ are the time integrals of the nonlocal Volterra and local generator terms. This is a decomposition by \emph{generator terms}, not an evolution of three independent signals: both integral components are evaluated using the full adjoint $P_t$.

For a region $B$ of positive length define
\begin{align}
J_\alpha(B)&:=\frac1{|B|}\E\|\Pi_BP_\alpha\|_{\mathcal X}^2,
\qquad \alpha\in\{\mathrm{res},\mathrm{cone},\mathrm{loc}\},\label{eq:regional-channel-energies}\\
\Xi(B)&:=\frac{2}{|B|}\sum_{\alpha<\beta}
\operatorname{Re}\E\langle\Pi_BP_\alpha,\Pi_BP_\beta\rangle_{\mathcal X}.
\label{eq:regional-channel-cross-terms}
\end{align}
Then
\begin{equation}\label{eq:exact-regional-influence-decomposition}
\mathcal A_s(B):=\frac1{|B|}\int_B I_{\vartheta_s}(p)\dd p
=J_{\rm res}(B)+J_{\rm cone}(B)+J_{\rm loc}(B)+\Xi(B).
\end{equation}

\begin{corollary}[Exact regional characterization]\label{cor:exact-double-deficit}
For $E_L=[0,\delta)$, $E_M=[\delta,1-\delta)$, and $E_R=[1-\delta,1]$, Lost-in-the-Middle holds if and only if
\begin{align}
\sum_\alpha[J_\alpha(E_L)-J_\alpha(E_M)]&>\Xi(E_M)-\Xi(E_L),\label{eq:exact-left-middle-channel-condition}\\
\sum_\alpha[J_\alpha(E_R)-J_\alpha(E_M)]&>\Xi(E_M)-\Xi(E_R).\label{eq:exact-right-middle-channel-condition}
\end{align}
This is an exact algebraic characterization, not an independently verifiable mechanism theorem.
\end{corollary}

For estimable sufficient conditions, suppose that for each region $B$ and pair $\alpha<\beta$ one has a correlation bound
\begin{equation}\label{eq:channel-correlation-bound}
\left|\frac1{|B|}\operatorname{Re}\E\langle\Pi_BP_\alpha,\Pi_BP_\beta\rangle\right|
\le \rho_{\alpha\beta,B}\sqrt{J_\alpha(B)J_\beta(B)},
\qquad 0\le\rho_{\alpha\beta,B}\le1.
\end{equation}
Cauchy--Schwarz always permits $\rho_{\alpha\beta,B}=1$; smaller empirical constants record weak channel alignment.

\begin{theorem}[Independently checkable channel-energy criterion]\label{thm:channel-energy-double-deficit}
Assume nonnegative upper bounds
\begin{equation}\label{eq:channel-upper-bounds}
J_\alpha(B)\le u_{\alpha,B},
\qquad \alpha\in\{\mathrm{res},\mathrm{cone},\mathrm{loc}\},\quad B\in\{E_L,E_M,E_R\},
\end{equation}
and lower bounds
\begin{equation}\label{eq:boundary-channel-lower-bounds}
J_{\rm cone}(E_L)\ge \ell_L,
\qquad
J_{\rm res}(E_R)\ge \ell_R.
\end{equation}
Let
\begin{equation}\label{eq:channel-cross-bound-gamma}
\Gamma_B:=2\sum_{\alpha<\beta}
\rho_{\alpha\beta,B}\sqrt{u_{\alpha,B}u_{\beta,B}},
\qquad
U_M:=u_{\rm res,E_M}+u_{\rm cone,E_M}+u_{\rm loc,E_M}.
\end{equation}
The quantities $u_{\rm cone,E_M}$ and $u_{\rm res,E_M}$ are middle-channel upper bounds, while $u_{\rm loc,B}$ supplies an explicit local-channel smallness control. If
\begin{equation}\label{eq:checkable-double-deficit-conditions}
\ell_L>U_M+\Gamma_M+\Gamma_L,
\qquad
\ell_R>U_M+\Gamma_M+\Gamma_R,
\end{equation}
then
\[
\overline m_L^{\rm GD}>\overline m_M^{\rm GD},
\qquad
\overline m_R^{\rm GD}>\overline m_M^{\rm GD},
\]
and hence $\Delta_\delta>0$, $\mathsf{LIM}_\delta>0$, and $\mathsf{SC}_\delta>0$.
\end{theorem}

\begin{proof}
Equation~\eqref{eq:channel-correlation-bound} and the upper bounds imply $|\Xi(B)|\le\Gamma_B$. Therefore
\[
\mathcal A_s(E_L)\ge \ell_L-\Gamma_L,
\qquad
\mathcal A_s(E_R)\ge \ell_R-\Gamma_R,
\]
whereas
\[
\mathcal A_s(E_M)\le U_M+\Gamma_M.
\]
The two inequalities in \eqref{eq:checkable-double-deficit-conditions} give strict left--middle and right--middle separation. Division by the common positive normalizer $Z_s$ preserves these inequalities.
\end{proof}

The constants in Theorem~\ref{thm:channel-energy-double-deficit} can be estimated from finite-token channel vectors or bounded analytically. The criterion can be conservative when only the universal choice $\rho=1$ is available, but unlike Corollary~\ref{cor:exact-double-deficit}, it does not assume the desired regional differences themselves.

\subsection{Explicit counterexamples to unconditional shape claims}\label{subsec:counterexamples}
The following finite-dimensional examples isolate why each structural slogan needs additional hypotheses.
\begin{enumerate}[label=(\roman*),leftmargin=2em]
\item \textbf{Causal masking without primacy.}  Take a causally masked attention block with $W_V=0$ and a tokenwise identity residual path.  The causal mask is present, but $D\mathrm{Attn}=0$.  For a spatially uniform terminal adjoint, $P_0=P_T$ is uniform, so no left advantage occurs.
\item \textbf{Residual connections without recency.}  Take the identity residual network $X_{k+1}=X_k$ with a uniformly distributed teacher-forced loss.  Then $P_0=P_T$ and the influence profile is uniform.  The residual identity transmits terminal geometry but does not create a right bias.
\item \textbf{Equal observability without equal influence.}  With two positions, identity dynamics, and identical observation maps, $G_1=G_2=I$.  If the task covectors are $P_T(1)=2e_1$ and $P_T(2)=e_1$, then the influence energies are $4$ and $1$ despite exactly equal observability Gramians.
\item \textbf{Negative cross terms can destroy a U-shape.}  On three representative regions let
$P_{\rm res}=(1,\tfrac12,1)$ and
$P_{\rm cone}=(-1,\tfrac12,-1)$ in a scalar feature space.  Each channel separately has energies $(1,\tfrac14,1)$, but their sum has energy $(0,1,0)$.  Boundary anti-correlation reverses the channel-energy conclusion, demonstrating why the cross terms in \eqref{eq:exact-regional-influence-decomposition} cannot be discarded.
\end{enumerate}

\subsection{Gradient-flow remedies}\label{sec:remedies}
The deterministic formulation suggests remedies that act directly on the influence density rather than indirectly on retrieval scores. Let $\nu$ be a target positional density selected from the intended deployment query distribution, task semantics, or a constrained robustness objective.  The uniform choice $\nu\equiv1$ is a neutral diagnostic default. For the continuum squared-density penalty, assume \(\nu\in L^2(0,1)\), the $L^4$ adjoint-sensitivity hypothesis \eqref{eq:adjoint-l4-sensitivity-assumption}, and the nondegeneracy condition \eqref{eq:influence-normalization-away-from-zero}; Proposition \ref{prop:influence-l2-differentiability} then makes the following objective finite and differentiable.  Without these extra hypotheses, the squared penalty is asserted only at finite token length (or after a stated spatial smoothing). A direct penalty is
\begin{equation}\label{eq:gd-target-penalty}
R_{\rm bal}(\vartheta)
=
R(\vartheta)
+
\frac{\beta}{2}
\int_0^1
\left(
\frac{I_\vartheta(p)}{\int_0^1 I_\vartheta(q)\dd q}
-\nu(p)
\right)^2
\dd p .
\end{equation}
A finite-token implementation is the following influence-balancing procedure.
\smallskip
\noindent\textbf{Algorithm 1: target influence-density balancing.}
\begin{enumerate}[leftmargin=2em,label=\arabic*.]
\item Choose a target density \(\nu_i\approx\nu(p_i)\), a regularization strength \(\beta>0\), and a batch \(\mathcal B\).
\item For each \((X_0,Y)\in\mathcal B\), run the forward residual Transformer and compute the usual task loss \(\mathcal L_{\mathcal B}(\vartheta)\).
\item Backpropagate the terminal adjoint through depth and record the tokenwise input-adjoint energies
\[
E_b(p_i):=\|P_{0,b}^{\vartheta}(p_i)\|^2,
\qquad (X_0,Y)=b\in\mathcal B .
\]
\item Estimate and normalize the positional influence by
\[
\widehat I_\vartheta(p_i):=\frac1{|\mathcal B|}\sum_{b\in\mathcal B}E_b(p_i),
\qquad
\widehat m_\vartheta(p_i):=
\frac{\widehat I_\vartheta(p_i)}{\sum_j\Delta p_j\widehat I_\vartheta(p_j)+\epsilon_{\rm inf}},
\qquad \epsilon_{\rm inf}>0.
\]
\item Minimize the empirical balanced objective
\begin{equation}\label{eq:discrete-target-influence-penalty}
\widehat R_{\rm bal}(\vartheta)
:=\mathcal L_{\mathcal B}(\vartheta)
+\frac{\beta}{2}\sum_i\Delta p_i
\bigl(\widehat m_\vartheta(p_i)-\nu_i\bigr)^2 .
\end{equation}
\end{enumerate}
Differentiating the penalty through $\widehat m_\vartheta$ requires derivatives of an input gradient with respect to model parameters, hence second-order automatic differentiation or Hessian--vector products. If the gradient through $\widehat m_\vartheta$ is stopped, the penalty contributes no gradient to the same optimization objective. A stop-gradient version must therefore be implemented as an outer loop: first measure $\widehat m$, then update separate positional weights as in Algorithm 2, and only then optimize the task loss. The stabilizer $\epsilon_{\rm inf}$ prevents division by nearly vanishing total adjoint energy.

\begin{remark}
    The additional term in \eqref{eq:gd-target-penalty} generally introduces an intentional \emph{regularization bias} relative to the minimizer of the original task risk.  To make this precise, define
\begin{equation}\label{eq:influence-discrepancy-functional}
\Phi_\nu(\vartheta)
:=\frac12\int_0^1\bigl(m_\vartheta(p)-\nu(p)\bigr)^2\dd p,
\qquad
m_\vartheta(p):=\frac{I_\vartheta(p)}{\int_0^1I_\vartheta(q)\dd q},
\end{equation}
so that \(R_{\rm bal}=R+\beta\Phi_\nu\).  Let \(\vartheta_0\) be an isolated local minimizer of \(R\), assume that
\(H_0:=\nabla^2R(\vartheta_0)\) is nonsingular, and let \(\vartheta_\beta\) denote the nearby stationary point of \(R_{\rm bal}\).  The implicit-function theorem gives
\begin{equation}\label{eq:influence-regularization-bias-expansion}
\vartheta_\beta-\vartheta_0
=-\beta H_0^{-1}\nabla\Phi_\nu(\vartheta_0)
+O(\beta^2).
\end{equation}
Hence the balanced objective changes the fitted parameter at first order unless
\(\nabla\Phi_\nu(\vartheta_0)=0\), which occurs, for example, when the original optimum already matches the target influence profile locally.  Under local \(\mu\)-strong convexity of \(R\) and a bound \(\|\nabla\Phi_\nu\|\le G\), one obtains the stability estimate
\begin{equation}\label{eq:influence-regularization-parameter-shift-bound}
\|\vartheta_\beta-\vartheta_0\|
\le \frac{\beta G}{\mu},
\end{equation}
showing explicitly how the bias is controlled by \(\beta\).

This is objective bias, not bias of the population gradient: exact differentiation of \eqref{eq:gd-target-penalty} gives the true gradient of the modified risk.  A separate finite-batch estimation bias can arise in \eqref{eq:discrete-target-influence-penalty}, because normalization is nonlinear and generally
\[
\E\!\left[
\frac{\widehat I_\vartheta(p_i)}
{\sum_j\Delta p_j\widehat I_\vartheta(p_j)+\epsilon_{\rm inf}}
\right]
\ne
\frac{I_\vartheta(p_i)}
{\sum_j\Delta p_j I_\vartheta(p_j)+\epsilon_{\rm inf}}.
\]
Large batches, moving-average estimates, and an independently estimated denominator reduce this ratio-estimation effect.  In applications, \(\beta\) should therefore be selected by a task-performance-versus-balance tradeoff, annealed toward zero when asymptotic task-risk consistency is required, or replaced by a constrained formulation \(\Phi_\nu(\vartheta)\le\tau\) with a tolerable imbalance level \(\tau\).
\end{remark}

A mask-aware remedy is to compensate the causal-cone asymmetry by a position weight \(w(p)>0\) in the loss:
\begin{equation}\label{eq:weighted-cross-entropy}
\mathcal\ell_w(Y,H(X))
=
-\int_0^1 w(p)Y(p)^\top\log\operatorname{softmax}(H(X)(p))\dd p .
\end{equation}
The idealized choice is to increase \(w\) on positions whose influence density is too small. A practical finite-token update can be written as follows.
\smallskip
\noindent\textbf{Algorithm 2: mask-aware positional reweighting.}
\begin{enumerate}[leftmargin=2em,label=\arabic*.]
\item Fix a target density \(\nu_i\) and initialize positive weights \(w_i^{(0)}>0\), for example \(w_i^{(0)}=1\).
\item At training epoch \(n\), estimate \(\widehat m^{(n)}(p_i)\) using the adjoint-energy procedure of Algorithm 1.
\item Update the weights by a damped multiplicative rule
\begin{equation}\label{eq:mask-aware-weight-update}
\widetilde w_i^{(n+1)}
=\operatorname{clip}_{[w_{\min},w_{\max}]}
\left(
 w_i^{(n)}\exp\{\eta_w(\nu_i-\widehat m^{(n)}(p_i))\}
\right),
\qquad
w_i^{(n+1)}=
\frac{\widetilde w_i^{(n+1)}}{\sum_j\Delta p_j\widetilde w_j^{(n+1)}}.
\end{equation}
\item Train the next epoch using the discrete weighted loss
\begin{equation}\label{eq:discrete-weighted-cross-entropy}
\widehat{\mathcal\ell}_w
=-\sum_i\Delta p_i\,w_i^{(n+1)}Y_i^\top
\log\operatorname{softmax}(H(X)_i).
\end{equation}
\end{enumerate}
The update \eqref{eq:mask-aware-weight-update} increases the loss weight where the measured influence \(\widehat m^{(n)}\) is below the target and decreases it where influence is excessive. The renormalization in \eqref{eq:mask-aware-weight-update} fixes the total positional mass of the loss.
The update is an outer-loop response rule rather than a gradient step on \(\mathsf{LIM}_\delta\).  To express this distinction, let \(\vartheta^+(w)\) denote the parameter obtained after the next training stage using weights \(w\), and define the induced influence-response map
\begin{equation}\label{eq:influence-response-map}
\mathcal M(w):=m_{\vartheta^+(w)}^{\rm GD},
\qquad
\mathcal L_\delta(w):=\mathsf{LIM}_\delta(\vartheta^+(w)).
\end{equation}
Ignoring clipping for a first-order calculation and writing
\(d_i^{(n)}:=\nu_i-\widehat m^{(n)}(p_i)\), normalization of the multiplicative step gives
\begin{equation}\label{eq:normalized-multiplicative-first-order-step}
w_i^{(n+1)}-w_i^{(n)}
=\eta_w w_i^{(n)}
\left(
 d_i^{(n)}-
 \sum_j\Delta p_j w_j^{(n)}d_j^{(n)}
\right)
+O(\eta_w^2).
\end{equation}
Even though this direction increases relative weight where the current influence is below target, the next influence is \(\mathcal M(w^{(n+1)})\), not the current \(\widehat m^{(n)}\).  Consequently,
\begin{equation}\label{eq:lim-outer-loop-first-order-change}
\mathcal L_\delta(w^{(n+1)})-\mathcal L_\delta(w^{(n)})
=\nabla_w\mathcal L_\delta(w^{(n)})^\top
\bigl(w^{(n+1)}-w^{(n)}\bigr)
+O\!\left(\|w^{(n+1)}-w^{(n)}\|^2\right),
\end{equation}
and the sign of the leading term is not determined by \eqref{eq:mask-aware-weight-update}.  A monotone decrease would require an additional response-alignment condition, for example
\begin{equation}\label{eq:lim-response-alignment-condition}
\nabla_w\mathcal L_\delta(w^{(n)})^\top
\bigl(w^{(n+1)}-w^{(n)}\bigr)
\le -c\|w^{(n+1)}-w^{(n)}\|^2
\end{equation}
for some \(c>0\), or a suitable contraction property of \(\mathcal M\).  Neither convexity of the original task risk nor such an alignment property is assumed here.  Therefore Algorithm 2 is a diagnostic-driven heuristic: clipping, damping, and validation can stabilize it, but no monotone decrease of \(\mathsf{LIM}_\delta\) is claimed.

\begin{remark}
The weighting in \eqref{eq:weighted-cross-entropy} changes the positional measure under which prediction error is minimized.  If
\begin{equation}\label{eq:positionwise-risk-density}
r_\vartheta(p)
:=\E\!\left[-Y(p)^\top
\log\operatorname{softmax}(H_\vartheta(X)(p))\right],
\end{equation}
then the unweighted and weighted population risks are
\begin{equation}\label{eq:weighted-risk-change-of-measure}
R(\vartheta)=\int_0^1r_\vartheta(p)\dd p,
\qquad
R_w(\vartheta)=\int_0^1w(p)r_\vartheta(p)\dd p.
\end{equation}
Unless \(w\) is constant, the two objectives need not have the same minimizer when parameters are shared across positions or the model class is misspecified.  Thus positional reweighting generally introduces an intentional bias toward the positions receiving larger weight.  This effect is absent in the idealized unrestricted, well-specified setting in which the model can realize the Bayes conditional distribution separately at every position: since cross-entropy is a strictly proper scoring rule, multiplication by any positive \(w(p)\) leaves the pointwise Bayes predictor unchanged and changes only the relative statistical efficiency across positions.

The weighting is not biased relative to a desired deployment risk when it is used as importance weighting.  If training positions are distributed according to a density \(\pi(p)>0\) and the target deployment density is \(\tau(p)\), choosing
\begin{equation}\label{eq:positional-importance-weight}
w(p)=\frac{\tau(p)}{\pi(p)}
\end{equation}
makes \(R_w\) equal to the target-position risk, provided the usual support and integrability conditions hold.  For the original uniform-position risk, however, any nonconstant adaptive weight changes the estimand.  Clipping in \eqref{eq:mask-aware-weight-update} controls variance and extreme optimization pressure, while normalization controls scale but does not remove this change of estimand.  Accordingly, the weighted objective should be evaluated against the original unweighted validation risk as well as the positional-balance diagnostic.
\end{remark}

A residual-aware diagnostic can be formulated through a finite-token observability Gramian.  Stack the hidden states at layer \(k\) as a vector in \(\R^{Ld_x}\), and let
\begin{equation}\label{eq:finite-token-position-injection}
E_i:\R^{d_x}\longrightarrow\R^{Ld_x}
\end{equation}
be the canonical injection that places a feature perturbation in token block \(i\) and sets all other token blocks to zero.  If \(\Phi_{k\leftarrow0}:=D_{X_0}X_k\) denotes the tangent propagator of the finite residual network from the input to layer \(k\), define
\begin{equation}\label{eq:finite-token-position-jacobian}
J_{k\leftarrow0,i}:=\Phi_{k\leftarrow0}E_i
\in\mathcal L(\R^{d_x},\R^{Ld_x}).
\end{equation}
Thus \(J_{k\leftarrow0,i}\xi\) is the full layer-\(k\) hidden-state variation generated by an input perturbation \(\xi\) applied only at token \(i\).

Let \(C_k:\R^{Ld_x}\to\R^{r_k}\) be a chosen linear observation map.  It may, for example, select a collection of token coordinates, linearize an intermediate readout, or retain only the final-token representation.
The observation protocol must be fixed before positions are compared. A task-aligned choice is $C_k=DH_k(X_k)$ for an auxiliary readout $H_k$, or a fixed final-query selector when the evaluation task reads only the last token. Different $C_k$ produce different Gramians; observability equalization is therefore relative to the selected task output and is not an intrinsic model property. A non-task-aligned coordinate selector remains a legitimate structural diagnostic, but it should not be interpreted as equalization of predictive influence.
The positionwise finite-horizon observability Gramian is
\begin{equation}\label{eq:finite-token-observability-gramian}
G_i:=\sum_{k=0}^{M-1}\varepsilon\,
J_{k\leftarrow0,i}^*C_k^*C_kJ_{k\leftarrow0,i}
\in\R^{d_x\times d_x}.
\end{equation}
It is symmetric positive semidefinite and satisfies, for every \(\xi\in\R^{d_x}\),
\begin{equation}\label{eq:finite-token-observability-quadratic-form}
\xi^*G_i\xi
=\sum_{k=0}^{M-1}\varepsilon\,
\|C_kJ_{k\leftarrow0,i}\xi\|^2.
\end{equation}
Hence \(G_i\) measures how strongly perturbations inserted at token \(i\) remain visible through the selected observations over depth.  Equalizing the entire matrices leads to
\begin{equation}\label{eq:gd-observability-penalty}
\Omega_{\rm obs}
:=\sum_{i=1}^{L}\Delta p_i\|G_i-\overline G\|_{\rm HS}^2,
\qquad
\overline G:=\sum_{i=1}^{L}\Delta p_iG_i,
\qquad \Delta p_i=\frac1L.
\end{equation}
The matrix penalty controls directional observability, since
\(\|G_i-\overline G\|_{\rm HS}^2\) compares all eigen-directions, not only total energy.  Forming every \(G_i\), however, generally requires \(d_x\) tangent propagations per position and is therefore practical only for small hidden dimension or restricted token subsets.

For large models, use randomized trace estimation.  Let \(\xi_1,\ldots,\xi_{N_{\rm probe}}\) be independent isotropic probes satisfying
\begin{equation}\label{eq:isotropic-observability-probes}
\E\xi_r=0,
\qquad
\E[\xi_r\xi_r^*]=I_{d_x};
\end{equation}
standard Gaussian or Rademacher vectors are admissible.  Define
\begin{equation}\label{eq:discrete-observability-trace-estimator}
\widehat g_i
:=\frac1{N_{\rm probe}}\sum_{r=1}^{N_{\rm probe}}
\sum_{k=0}^{M-1}\varepsilon\,
\|C_kJ_{k\leftarrow0,i}\xi_r\|^2.
\end{equation}
Using \eqref{eq:finite-token-observability-quadratic-form} and the trace identity
\(\E[\xi_r^*G_i\xi_r]=\operatorname{tr}(G_i\E[\xi_r\xi_r^*])\), one obtains
\begin{equation}\label{eq:observability-trace-unbiasedness}
\E\widehat g_i=\operatorname{tr}G_i.
\end{equation}
Thus \(g_i:=\operatorname{tr}G_i\) measures the total observability of an isotropic perturbation at token \(i\).  A scalar trace-balancing target is
\begin{equation}\label{eq:true-trace-observability-penalty}
\Omega_{\rm obs}^{\rm tr}
:=\sum_{i=1}^{L}\Delta p_i(g_i-\overline g_{\rm true})^2,
\qquad
\overline g_{\rm true}:=\sum_{i=1}^{L}\Delta p_i g_i,
\end{equation}
and its computational surrogate is
\begin{equation}\label{eq:discrete-observability-penalty}
\widehat\Omega_{\rm obs}^{\rm tr}
:=\sum_{i=1}^{L}\Delta p_i(\widehat g_i-\overline g)^2,
\qquad
\overline g:=\sum_{i=1}^{L}\Delta p_i\widehat g_i.
\end{equation}
Although each \(\widehat g_i\) is unbiased, the squared empirical penalty is generally upward biased because of probe variance.  This does not prevent its use as a stochastic objective; using common probes across positions, increasing \(N_{\rm probe}\), or averaging estimates across iterations reduces its variance.  The trace penalty equalizes total local observability but cannot detect anisotropy between feature directions; the Hilbert--Schmidt penalty \eqref{eq:gd-observability-penalty} should therefore be used only when matrix-valued estimates are actually available.

\smallskip
\noindent\textbf{Algorithm 3: randomized observability balancing.}
\begin{enumerate}[leftmargin=2em,label=\arabic*.]
\item Choose observation maps \(C_k\), a set of monitored layers \(\mathcal K\subset\{0,\ldots,M-1\}\), monitored token positions \(\mathcal S\subset\{1,\ldots,L\}\), a probe count \(N_{\rm probe}\), and a regularization strength \(\lambda_{\rm obs}\ge0\).
\item Run the usual forward pass and retain the checkpoints needed to evaluate Jacobian--vector products through the residual layers.
\item Draw isotropic probes \(\xi_r\).  For each \(i\in\mathcal S\) and each probe, inject \(E_i\xi_r\) at the input and propagate it with forward-mode automatic differentiation to compute \(J_{k\leftarrow0,i}\xi_r\) for \(k\in\mathcal K\).
\item Accumulate
\[
\widehat g_i
=\frac1{N_{\rm probe}}\sum_{r=1}^{N_{\rm probe}}
\sum_{k\in\mathcal K}\varepsilon_k
\|C_kJ_{k\leftarrow0,i}\xi_r\|^2,
\]
form the weighted mean \(\overline g\), and evaluate \(\widehat\Omega_{\rm obs}^{\rm tr}\) from \eqref{eq:discrete-observability-penalty}.  If only a diagnostic is required, report the profile \(i\mapsto\widehat g_i\) and stop here.
\item For regularized training, minimize
\begin{equation}\label{eq:observability-balanced-training-objective}
\widehat R_{\rm obs}(\vartheta)
:=\mathcal L_{\mathcal B}(\vartheta)
+\frac{\lambda_{\rm obs}}2\widehat\Omega_{\rm obs}^{\rm tr}(\vartheta).
\end{equation}
Since \(J_{k\leftarrow0,i}\) itself depends on \(\vartheta\), differentiating the observability term requires differentiating through Jacobian--vector products, i.e. mixed second-order derivatives.  A stop-gradient implementation remains a diagnostic and does not regularize the current parameter update.
\item To control cost, subsample \(\mathcal S\) and \(\mathcal K\), reuse common probes across positions, and update the observability estimate less frequently than the task-loss gradient.  Validate both the original task risk and the positional influence profile, since equal observability is a structural surrogate rather than a guaranteed decrease of \(\mathsf{LIM}_\delta\).
\end{enumerate}

\begin{remark}[Bias induced by Algorithm 3]\label{rem:algorithm3-bias}
Algorithm 3 can introduce two distinct notions of bias, and they should not be conflated.  The first is \emph{objective bias}.  If the randomized observability term is used only as a diagnostic, the fitted parameter is unchanged.  If it is included in the training objective \eqref{eq:observability-balanced-training-objective}, however, the population target becomes
\[
R_{\lambda_{\rm obs}}(\vartheta)
:=R(\vartheta)+\frac{\lambda_{\rm obs}}2
\Omega_{\rm obs}^{\rm tr}(\vartheta),
\]
which generally has a different minimizer from the original task risk.  More precisely, suppose that \(\vartheta_0\) is an isolated minimizer of \(R\), that
\(H_0:=\nabla^2R(\vartheta_0)\) is positive definite, and that
\(\Omega_{\rm obs}^{\rm tr}\) is twice continuously differentiable near
\(\vartheta_0\).  The implicit-function theorem then gives, for sufficiently small
\(\lambda_{\rm obs}\),
\begin{equation}\label{eq:algorithm3-objective-bias-expansion}
\vartheta_{\lambda_{\rm obs}}-\vartheta_0
=-\frac{\lambda_{\rm obs}}2
H_0^{-1}\nabla\Omega_{\rm obs}^{\rm tr}(\vartheta_0)
+O(\lambda_{\rm obs}^2).
\end{equation}
Thus the leading displacement vanishes only when the unregularized solution is already stationary for the observability penalty, or when
\(\lambda_{\rm obs}=0\).  This is an intentional structural bias, analogous to the regularization bias discussed for Algorithm~1.

The second issue is \emph{estimation bias}.  Let
\(d_i:=\Delta p_i\), \(D:=\operatorname{diag}(d_1,\ldots,d_L)\),
\(d=(d_1,\ldots,d_L)^\top\), and
\(Q:=D-dd^\top\succeq0\).  Writing
\(g=(g_1,\ldots,g_L)^\top\) and
\(\widehat g=(\widehat g_1,\ldots,\widehat g_L)^\top\), the true and empirical trace penalties are
\[
\Omega_{\rm obs}^{\rm tr}=g^\top Qg,
\qquad
\widehat\Omega_{\rm obs}^{\rm tr}=\widehat g^\top Q\widehat g.
\]
If \(\E\widehat g=g\) and
\(\Sigma:=\operatorname{Cov}(\widehat g)\), then
\begin{equation}\label{eq:algorithm3-penalty-estimator-bias}
\E\widehat\Omega_{\rm obs}^{\rm tr}
=\Omega_{\rm obs}^{\rm tr}+\operatorname{tr}(Q\Sigma)
\ge \Omega_{\rm obs}^{\rm tr}.
\end{equation}
Hence unbiased trace estimates do not produce an unbiased squared balancing penalty.  Common probes across positions can reduce
\(\operatorname{tr}(Q\Sigma)\), but do not in general remove it.  An unbiased cross-fitted alternative uses two independent probe batches:
\begin{equation}\label{eq:algorithm3-crossfit-unbiased-penalty}
\widetilde\Omega_{\rm obs}^{\rm tr}
:=(\widehat g^{(a)})^\top Q\widehat g^{(b)},
\qquad
\E\widetilde\Omega_{\rm obs}^{\rm tr}
=\Omega_{\rm obs}^{\rm tr}.
\end{equation}
The cross-fitted estimator need not be nonnegative on an individual batch and can have substantially larger gradient variance because two independent noisy profiles are multiplied. It is therefore intended primarily as an unbiased diagnostic or theoretical comparison. For routine training, the nonnegative biased estimator \eqref{eq:discrete-observability-penalty}, with common probes and temporal averaging, is usually the safer objective; using the cross-fitted form as a training loss requires explicit variance control and validation.

Subsampling monitored positions or layers creates a further approximation bias unless the sampled terms are corrected by their inclusion probabilities.  Inverse-probability weighting preserves the desired trace sum when the sampling probabilities are known and positive; uncorrected subsampling instead optimizes the observability of the monitored subset.  Finally, stopping the gradient through the Jacobian--vector products does not create a biased gradient of the same objective; it removes the observability contribution altogether and leaves Algorithm~3 as a diagnostic.  Consequently, the regularized model should be evaluated using the original unweighted task risk, the true or high-probe observability profile, and the positional influence profile.
\end{remark}

\subsection{Computational complexity of Algorithms 1 and 3}\label{subsec:algorithm-complexity}
Let $L$ be sequence length, $M$ depth, $d_x$ hidden dimension, $S=|\mathcal S|$ monitored positions, $K=|\mathcal K|$ monitored layers, and $R=N_{\rm probe}$. Write $C_{\rm layer}(L,d_x)$ for the cost of one residual layer; for dense attention it is $O(L^2d_x+Ld_x^2)$ up to head constants.

Algorithm~1 obtains all positional input-adjoint energies from one ordinary reverse pass, so diagnostic evaluation costs $O(MC_{\rm layer})$ time and standard activation storage $O(MLd_x)$ without checkpointing. Training through the normalized influence penalty requires mixed second-order differentiation or Hessian--vector products. Its asymptotic order remains $O(MC_{\rm layer})$ with a larger constant and additional graph retention; checkpointing can reduce memory toward $O(Ld_x\log M)$ at the cost of recomputation. Importantly, computing the complete input-gradient profile does not require $L$ separate backward passes.

The forward-mode implementation of Algorithm~3 propagates $SR$ input tangent directions. Its naive time is
\[
O\!\left(SR\,M C_{\rm layer}(L,d_x)\right),
\]
with accumulation at the $K$ monitored layers costing an additional $O(SRKd_x)$ and with probes/positions batchable when memory permits. Exact matrix Gramians replace $R$ by $d_x$. Reverse mode is preferable when the total number of independent observation cotangents, approximately $R\sum_{k\in\mathcal K}r_k$, is smaller than the number $SR$ of input directions; forward mode is preferable in the opposite regime, especially for a small monitored position set. Subsampling positions or layers lowers cost in proportion to $S$ or $K$ but changes the estimand unless inclusion-probability correction is used.

In summary, the three proposed remedies are: a second-order influence penalty, an outer-loop position reweighting rule, and an observability diagnostic or regularizer. Their optimization effectiveness requires separate analysis and is not implied solely by the structural identities above.

\section{Controlled Finite-Token Simulations}\label{sec:simulations}

This section reports reproducible numerical experiments for the three remedies in
\Cref{sec:remedies}.  The experiments are deliberately small and mechanistic: they test whether each algorithm changes the mathematical quantity that it is designed to control.  They are not presented as a language-model benchmark or as evidence that the same hyperparameters transfer to a production Transformer.  In particular, Algorithm~2 and Algorithm~3 are evaluated together with the limitations already identified in
\eqref{eq:lim-response-alignment-condition} and Remark~\ref{rem:algorithm3-bias}.

\subsection{Simulation model and metrics}

We use \(L=48\) token positions, feature dimension \(d_x=4\), and
\(M=12\) residual steps with \(T=1\) and \(\varepsilon=T/M\).  Let
\begin{equation}\label{eq:simulation-causal-matrix}
C_{ij}:=\frac{\one_{\{j\le i\}}}{i},
\qquad 1\le i,j\le L,
\end{equation}
be the finite Ces\`aro causal-averaging matrix and let
\begin{equation}\label{eq:simulation-feature-mixing}
S=
\begin{pmatrix}
1 & 0.25 & 0 & 0\\
0.15 & 0.8 & 0.2 & 0\\
0 & 0.15 & 0.6 & 0.2\\
0 & 0 & 0.2 & 0.45
\end{pmatrix}.
\end{equation}
The linear residual propagator is
\begin{equation}\label{eq:simulation-residual-propagator}
B:=I_{Ld_x}+\varepsilon\alpha(C\otimes S),
\qquad \alpha=2,
\qquad \Phi:=B^M.
\end{equation}
Positive trainable positional gates are represented by
\begin{equation}\label{eq:simulation-input-gate}
D(u):=\operatorname{diag}(e^{u_1},\ldots,e^{u_L})\otimes I_{d_x}.
\end{equation}
They may be viewed as a simplified input-embedding or early residual scaling mechanism.  The terminal covector combines a final-token query with weak uniform supervision,
\begin{equation}\label{eq:simulation-terminal-covector}
r:=\left(\beta e_L+\frac{1-\beta}{L}\mathbf 1\right)\otimes v,
\qquad
\beta=0.85,
\qquad
v=(1,0.5,-0.3,0.2)^\top.
\end{equation}
The input adjoint and normalized positional influence masses are
\begin{equation}\label{eq:simulation-influence-profile}
p_0(u)=D(u)\Phi^*r,
\qquad
m_i(u)=
\frac{\|E_i^*p_0(u)\|^2}
{\sum_{j=1}^L\|E_j^*p_0(u)\|^2}.
\end{equation}
For Algorithm~2, the uniform part of the terminal positional covector is replaced by
\((1-\beta)L^{-1}w\), and the weights are updated by
\eqref{eq:mask-aware-weight-update}.

For Algorithm~3, every observation map selects the final token and retains all four features.  The exact finite-token Gramian is therefore
\begin{equation}\label{eq:simulation-observability-gramian}
G_i(u)=e^{2u_i}
\sum_{k=0}^{M-1}\varepsilon
E_i^*(B^k)^*C_{\rm out}^*C_{\rm out}B^kE_i,
\end{equation}
where \(C_{\rm out}\) extracts the final-token block.  We normalize its trace by
\begin{equation}\label{eq:simulation-normalized-observability}
q_i(u):=
\frac{\operatorname{tr}G_i(u)}
{\sum_j\operatorname{tr}G_j(u)}.
\end{equation}
The influence and observability imbalances are
\begin{equation}\label{eq:simulation-imbalance-metrics}
\mathcal E_{\rm inf}:=\frac1L\sum_{i=1}^L(Lm_i-1)^2,
\qquad
\mathcal E_{\rm obs}:=\frac1L\sum_{i=1}^L(Lq_i-1)^2.
\end{equation}
Regional averages use the discrete half-open sets corresponding to $E_L,E_M,E_R$. We set $\varepsilon_0=10^{-8}$ and report the discrete analogues of $\Delta_{0.2}$, $\mathsf{SC}_{0.2}$, and $\mathsf{LIM}_{0.2}$. All runs use NumPy and PyTorch in double precision with random seed $20260717$.

Since $G_i(u)=e^{2u_i}G_i(0)$, every eigenvalue is rescaled by the same positionwise factor and the condition number is exactly invariant:
\begin{equation}\label{eq:simulation-condition-number-invariance}
\kappa_0(G_i(u))=\kappa_0(G_i(0)).
\end{equation}
Thus this parameterization can balance traces while leaving directional anisotropy untouched.  For the stated $d_x=4$ matrices, the baseline regional means of the ordinary condition number are $171.1$, $36.7$, and $16.6$ on the left, middle, and right regions, respectively; the corresponding ranges are $[76.5,482.8]$, $[20.8,70.2]$, and $[1.03,20.2]$.  These values are unchanged after Algorithm~3.  This spectral diagnostic is more informative than the trace plot alone and confirms that the toy objective does not establish matrix observability balance.

\subsection{Optimization protocols}

Algorithm~1 minimizes
\begin{equation}\label{eq:simulation-algorithm1-objective}
J_1(u)=\frac{\gamma_1}{2L}\|u\|^2
+\frac{\lambda_1}{2}\mathcal E_{\rm inf}(u),
\qquad
\gamma_1=0.12,
\quad \lambda_1=15,
\end{equation}
for 400 Adam steps with learning rate \(0.04\).  The quadratic term is a task-fidelity proxy preventing unrestricted gate changes.  Algorithm~2 uses 160 outer-loop updates with
\(\eta_w=0.5\), clipping interval \([0.15,8]\), and unit-mean renormalization.  Algorithm~3 minimizes
\begin{equation}\label{eq:simulation-algorithm3-objective}
J_3(u)=\frac{\gamma_3}{2L}\|u\|^2
+\frac{\lambda_3}{2}\mathcal E_{\rm obs}(u),
\qquad
\gamma_3=1,
\quad \lambda_3=0.5,
\end{equation}
for 400 Adam steps with learning rate \(0.03\).  since \(d_x=4\) is small, this run uses the exact traces rather than stochastic probes; probe error is examined separately below.

\subsection{Running results}

The baseline exhibits both boundary channels: the left, middle, and right influence-density averages are respectively $3.2269$, $0.1417$, and $1.3479$.  The corresponding gap is $\Delta_{0.2}=1.2062$, the symmetric contrast is $\mathsf{SC}_{0.2}=0.8097$, and $\mathsf{LIM}_{0.2}=0.8949$ for $\varepsilon_0=10^{-8}$.  The table also reports the raw total adjoint energy $Z$, which is not determined by the normalized profile.

\begin{table}[htbp]
\centering
\caption{Algebraic finite-token results under the normalized discrete convention.  The raw energy is $Z=\sum_i\|E_i^*p_0\|^2$; the stabilizer is $\varepsilon_0=10^{-8}$.}
\label{tab:simulation-remedy-results}
\resizebox{\linewidth}{!}{%
\begin{tabular}{lrrrrrrrr}
\hline
Method & Left & Middle & Right & $\Delta_{0.2}$ & $\mathsf{SC}_{0.2}$ & $\mathsf{LIM}_{0.2}$ & $\mathcal E_{\rm inf}$ & $Z$\\
\hline
Baseline & 3.2269 & 0.1417 & 1.3479 & 1.2062 & 0.8097 & 0.8949 & 9.2503 & 4.1487\\
Algorithm 1 & 1.0163 & 0.9960 & 0.9958 & -0.0002 & -0.0001 & -0.0002 & 0.0002 & 9.7981\\
Algorithm 2 & 3.0303 & 0.1631 & 1.4804 & 1.3173 & 0.8015 & 0.8898 & 8.5348 & 3.8028\\
Algorithm 3 & 2.8192 & 0.6671 & 0.1794 & -0.4877 & -0.5761 & -2.7181 & 1.3952 & 1.7226\\
\hline
\end{tabular}}
\end{table}

Algorithm~1 nearly reaches the uniform target: $\mathcal E_{\rm inf}$ falls from $9.2503$ to $0.0002$, and the three regional averages become approximately equal.  Because a separate $u_i$ directly rescales each influence energy, this result only confirms the implementation on a directly controllable toy model; it does not verify removal of imbalance by shared Transformer parameters.

The raw normalization values reveal information hidden by the normalized profiles.  Algorithm~1 more than doubles $Z$ while flattening the shape, Algorithm~2 changes $Z$ only modestly, and Algorithm~3 reduces $Z$ to about $41.5\%$ of baseline.  A method can therefore appear balanced after normalization while substantially amplifying or suppressing total task sensitivity; both quantities must be interpreted together.

Algorithm~2 produces only a modest change.  Its influence imbalance decreases by approximately $7.7\%$, while $\mathsf{LIM}_{0.2}$ changes from $0.8949$ to $0.8898$.  The reason is visible in the response map: the adjustable uniformly supervised component has mass only \(1-\beta=0.15\), whereas the fixed final-token component and the causal propagator dominate the adjoint.  Thus the multiplicative rule points in the intended current-profile direction but encounters a poorly aligned model response.  This run numerically illustrates why
\eqref{eq:lim-response-alignment-condition}, rather than the sign of the weight discrepancy alone, is needed for a monotonic guarantee.

Algorithm~3 strongly reduces its scalar trace target in this directly rescalable parameterization: the observability imbalance falls from $39.4921$ to $0.3392$.  The normalized trace-density regional averages move from $(0.2905,0.0351,4.6041)$ to $(1.6292,0.8788,0.7343)$.  This is an algebraic check of trace rescaling, not evidence that training a shared-parameter Transformer equalizes observability matrices.  Equation~\eqref{eq:simulation-condition-number-invariance} shows that the directional condition numbers do not improve at all.  Its influence imbalance becomes $1.3952$, while the right average is pushed below the middle average and $\mathsf{LIM}_{0.2}=-2.7181$.  This is not a contradiction: Algorithm~3 equalizes a scalar structural surrogate, not the influence density or the full Gramian spectrum.
The result confirms the implementation of the scalar trace objective and the warning in Step~7 of Algorithm~3 that trace equalization need not monotonically improve $\mathsf{LIM}_\delta$.

\begin{figure}[htbp]
\centering
\includefigureorplaceholder[width=0.88\linewidth]{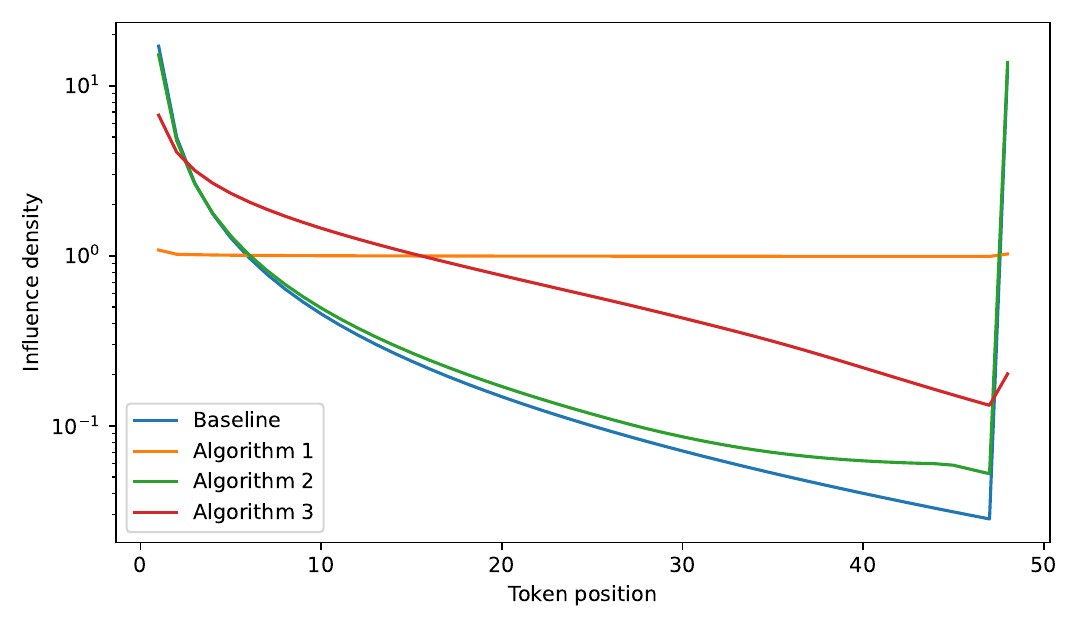}
\caption{Algebraic influence-density profiles.  Algorithm~1 directly rescales the profile toward uniformity, Algorithm~2 changes it only weakly, and Algorithm~3 over-corrects the right boundary.  The vertical scale is logarithmic.}
\label{fig:simulation-influence-profiles}
\end{figure}

\begin{figure}[htbp]
\centering
\includefigureorplaceholder[width=0.88\linewidth]{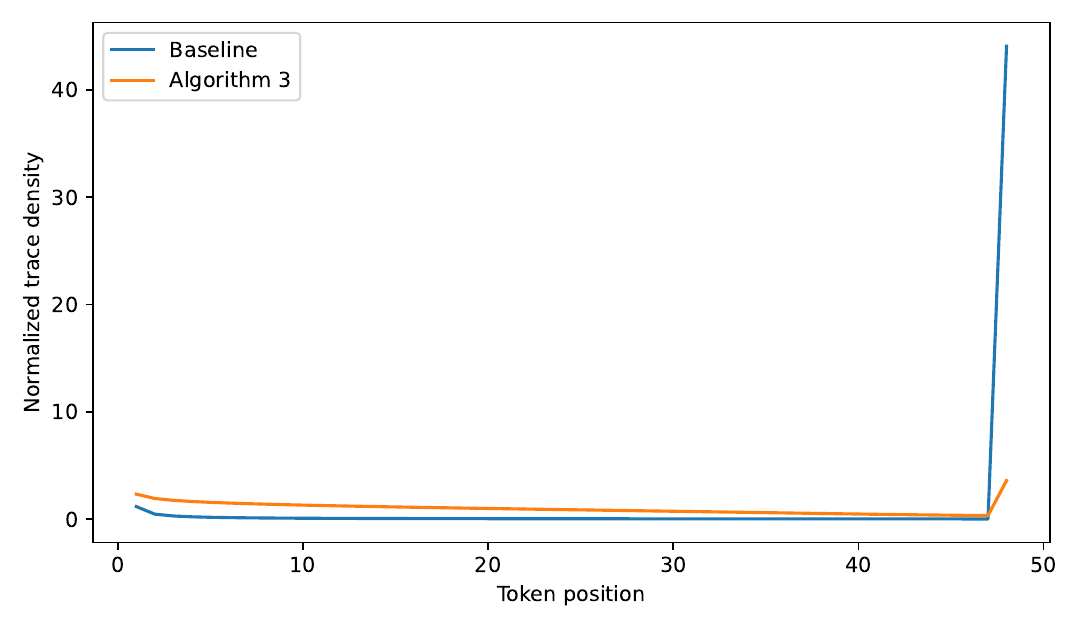}
\caption{Normalized trace-observability profile before and after the algebraic Algorithm~3 update.  Scalar trace balance does not imply eigenvalue or condition-number balance.}
\label{fig:simulation-observability-profiles}
\end{figure}

\begin{figure}[htbp]
\centering
\includefigureorplaceholder[width=0.88\linewidth]{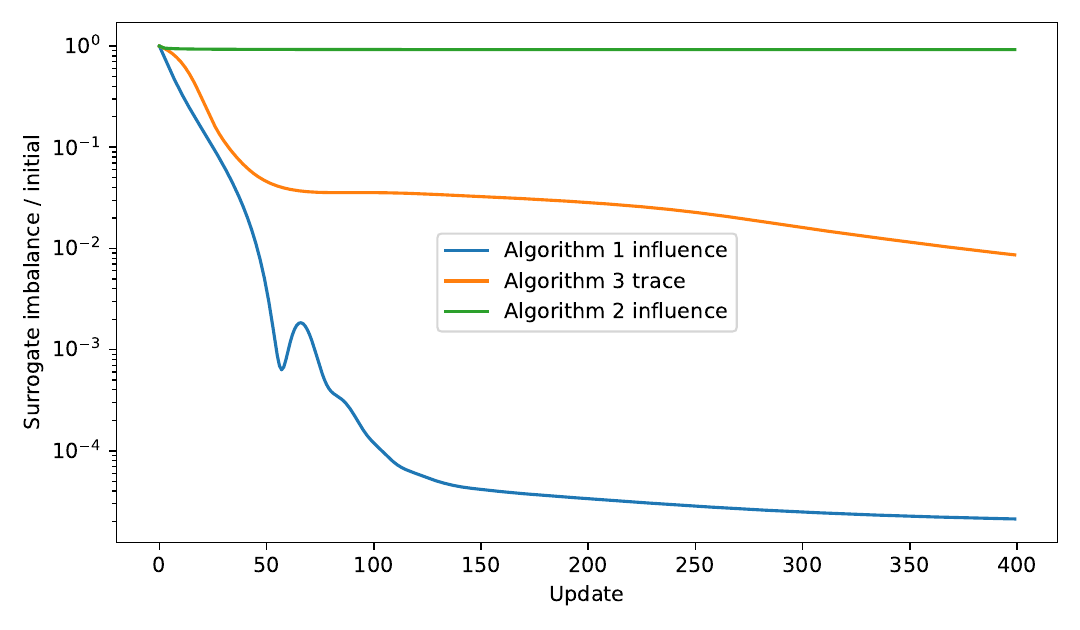}
\caption{Directly controlled surrogate imbalance divided by its initial value in the algebraic examples.}
\label{fig:simulation-convergence}
\end{figure}

\subsection{Probe bias and variance}

To check the stochastic implementation in Algorithm~3, we estimate the baseline trace profile with common Gaussian probes and repeat the experiment 1000 times.  For each probe count $N_{\rm probe}\in\{1,2,4,8,16,32\}$, we evaluate the empirical squared penalty in \eqref{eq:discrete-observability-penalty}.  The true unnormalized penalty is $0.3438$.  In the supplied reproducible run, the relative upward bias decreases from $44.0\%$ with one probe to $1.0\%$ with 32 probes, while the single-batch relative standard deviation decreases from $216.3\%$ to $24.4\%$.  The plotted dispersion is not the standard error of the 1000-repetition mean; that standard error is the displayed standard deviation divided by $\sqrt{1000}$.  These results agree with \eqref{eq:algorithm3-penalty-estimator-bias}: increasing the probe count reduces both the covariance term and stochastic variability, at the cost of more Jacobian--vector products.

\begin{figure}[htbp]
\centering
\includefigureorplaceholder[width=0.88\linewidth]{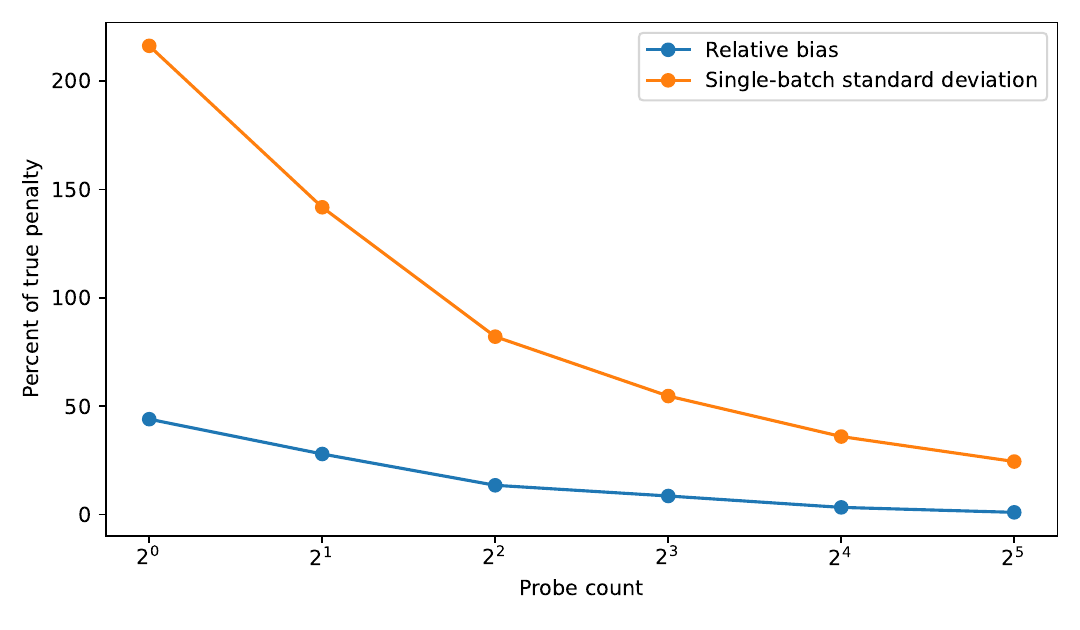}
\caption{Monte Carlo bias and single-batch standard deviation of the empirical squared observability penalty, expressed as percentages of the true penalty.  The plotted dispersion is not the standard error of the 1000-repetition mean.}
\label{fig:simulation-probe-bias}
\end{figure}

\subsection{Interpretation, reproducibility, and required Transformer validation}

The controlled examples support only three implementation-level statements: the directly differentiated influence penalty changes the directly rescaled influence profile; the outer-loop weighting rule can stall when its adjustable component is weak; and scalar trace balancing need not improve influence or directional observability.  They do not support claims about optimization through an actual Transformer, task preservation, or generalization.

A future empirical validation of the proposed training procedures should use at least a small causal Transformer on a synthetic long-context retrieval task and report: shared parameters rather than free positionwise gains; train and validation cross-entropy or retrieval accuracy; multiple random seeds with confidence intervals or standard errors; ablations over regularization strength, sequence length, depth, target density, and probe count; measured wall-clock and memory overhead; position-shift evaluation of both influence and retrieval performance; and direct empirical tests of the channel-energy sufficient conditions.  Until such an experiment is supplied, Algorithms~1--3 are proposals and the present section remains an algebraic controlled example.

The companion script \texttt{remedy\_simulations.py} constructs the stated matrices and regenerates the scalar tables and plots.  Reproducibility claims are limited to those algebraic outputs.

\section{Conclusion and limitations}\label{sec:conclusion}
The paper proves an analytic foundation for positional adjoint sensitivity in causal residual Transformers: well-posed masked depth dynamics, residual-to-flow and finite-token-to-Volterra approximation, readout-adjoint consistency under stated regularity, differentiability of the normalized influence density, its exact gradient-flow evolution, and an exact regional generator-term decomposition retaining covariance cross terms.

The paper does \emph{not} prove that causal masking and residual connections universally generate Lost-in-the-Middle. Primacy requires cross-data and cross-depth non-cancellation; recency requires a sufficiently right-biased terminal readout or task protocol and a moment-controlled interaction correction; a double deficit follows from Theorem~\ref{thm:channel-energy-double-deficit} only when measurable channel-energy and correlation bounds separate both boundaries from the middle. The regularizers are likewise proposals: their optimization response is not implied by the structural identities.

Open problems include deriving sharp correlation constants from architectural assumptions, replacing sufficient energy gaps by necessary-and-sufficient dynamical conditions, extending endpoint-aware finite-token estimates to weaker spatial and BV label regularity, obtaining moment bounds for heavy-tailed data, and testing whether the proposed diagnostics predict retrieval accuracy in production-scale long-context models. A further practical question is how to choose task-aligned observation maps that remain computationally tractable while preserving the causal interpretation of the readout.

\appendix

\section{Proof of the Residual-to-ODE Limit}\label{app:proof-resnet-ode}

\begin{proof}
Write
\[
G_\theta(t,X):=\zeta(t)\mathscr F_{\theta(t)}(X),
\qquad e_k:=\|X_k^\varepsilon-X_{t_k}^\theta\|_{\mathcal X}.
\]
The limiting flow and the discrete layer satisfy
\begin{align*}
X_{t_{k+1}}^\theta
&=X_{t_k}^\theta+
\int_{t_k}^{t_{k+1}}G_\theta(r,X_r^\theta)\dd r,\\
X_{k+1}^\varepsilon
&=X_k^\varepsilon+
\varepsilon\zeta(t_k)\mathscr F_{\theta_k^\varepsilon}
(X_k^\varepsilon)+d_k^\varepsilon.
\end{align*}
Using the state Lipschitz estimate from Proposition \ref{prop:transformer-lipschitz}, subtracting these identities gives
\begin{equation}\label{eq:convergent-control-error-recursion-pre}
e_{k+1}
\le(1+C\varepsilon)e_k
+\|d_k^\varepsilon\|_{\mathcal X}
+\tau_k,
\end{equation}
where
\[
\tau_k:=\int_{t_k}^{t_{k+1}}
\left\|
\zeta(t_k)\mathscr F_{\theta_k^\varepsilon}(X_{t_k}^\theta)
-\zeta(r)\mathscr F_{\theta(r)}(X_r^\theta)
\right\|_{\mathcal X}\dd r.
\]
Insert successively the terms with $(\zeta(r),\theta_k^\varepsilon,X_{t_k}^\theta)$ and $(\zeta(r),\theta(r),X_{t_k}^\theta)$.  Lipschitz continuity of $\zeta$, local Lipschitz dependence on $\theta$, and state Lipschitz continuity yield
\begin{align}
\tau_k
&\le C\varepsilon^2
+C\int_{t_k}^{t_{k+1}}
\|\theta_k^\varepsilon-\theta(r)\|\dd r
+C\int_{t_k}^{t_{k+1}}
\|X_r^\theta-X_{t_k}^\theta\|_{\mathcal X}\dd r\notag\\
&\le C\varepsilon^2
+C\int_{t_k}^{t_{k+1}}
\|\theta^\varepsilon(r)-\theta(r)\|\dd r.
\label{eq:convergent-control-local-defect}
\end{align}
The last inequality uses boundedness of the vector field, which gives
$\|X_r^\theta-X_{t_k}^\theta\|\le C|r-t_k|$.
Combining \eqref{eq:convergent-control-error-recursion-pre},
\eqref{eq:convergent-control-local-defect}, and
$\|d_k^\varepsilon\|\le C_{\rm split}\varepsilon^2$ gives
\[
e_{k+1}
\le(1+C\varepsilon)e_k+C\varepsilon^2
+C\int_{t_k}^{t_{k+1}}
\|\theta^\varepsilon(r)-\theta(r)\|\dd r.
\]
Discrete Gronwall and summation over $k$ imply
\begin{equation}\label{eq:convergent-control-grid-bound}
\max_{0\le k\le M}e_k
\le C_T\left(
\varepsilon+\|\theta^\varepsilon-\theta\|_{L^1(0,T)}
\right).
\end{equation}
For $t\in[t_k,t_{k+1})$, boundedness of the limiting vector field yields
$\|X_t^\theta-X_{t_k}^\theta\|\le C\varepsilon$.
Combining this with \eqref{eq:convergent-control-grid-bound} proves
\eqref{eq:resnet-ode-convergent-control-bound}.
The same argument with $\theta$ replaced by $\theta^\varepsilon$ removes the $L^1$ term and gives the stated shadowing interpretation, but the reference flow then depends on $\varepsilon$.
\end{proof}

\end{document}